%% file: main.tex
\documentclass{article}

\usepackage[main,final]{neurips_2026}
\setcitestyle{numbers,square,comma,sort&compress}

\usepackage[utf8]{inputenc}
\usepackage[T1]{fontenc}
\usepackage{amsmath,amssymb,amsfonts,amsthm}
\usepackage{booktabs}
\usepackage{array}
\usepackage{graphicx}
\usepackage{xcolor}
\usepackage{hyperref}
\usepackage{cleveref}
\usepackage{multirow}
\usepackage{algorithm}
\usepackage{algorithmic}
\usepackage{subcaption}
\usepackage{float}

\newtheorem{theorem}{Theorem}[section]
\newtheorem{proposition}{Proposition}[section]
\newtheorem{corollary}{Corollary}[section]
\newtheorem{lemma}{Lemma}[section]
\newtheorem{assumption}{Assumption}[section]

\title{Beyond Missing Rates: Rethinking
  Incomplete Multi-View Clustering
  with Protocol Divergence}

\input{authors}

\begin{document}
\maketitle


\input{section_abstract}


\input{section_1_intro}
\input{section_2_background}

\input{section_3_problem}

\input{section_4_theory}

\input{section_5_method}

\input{section_6_1_setup}

\input{section_6_2_main}

\input{section_6_3_ablation}

\input{section_7_discussion}
\label{end:maintext}

\begin{ack}
This work was supported by the National Natural Science Foundation of
China (No.\ U24A20323). We acknowledge support from the Center for HPC,
University of Electronic Science and Technology of China.
\end{ack}


\bibliographystyle{plainnat}
\bibliography{refs}


\newpage
\appendix
\raggedbottom

\begin{figure}[t]
  \centering
  \includegraphics[width=\linewidth]{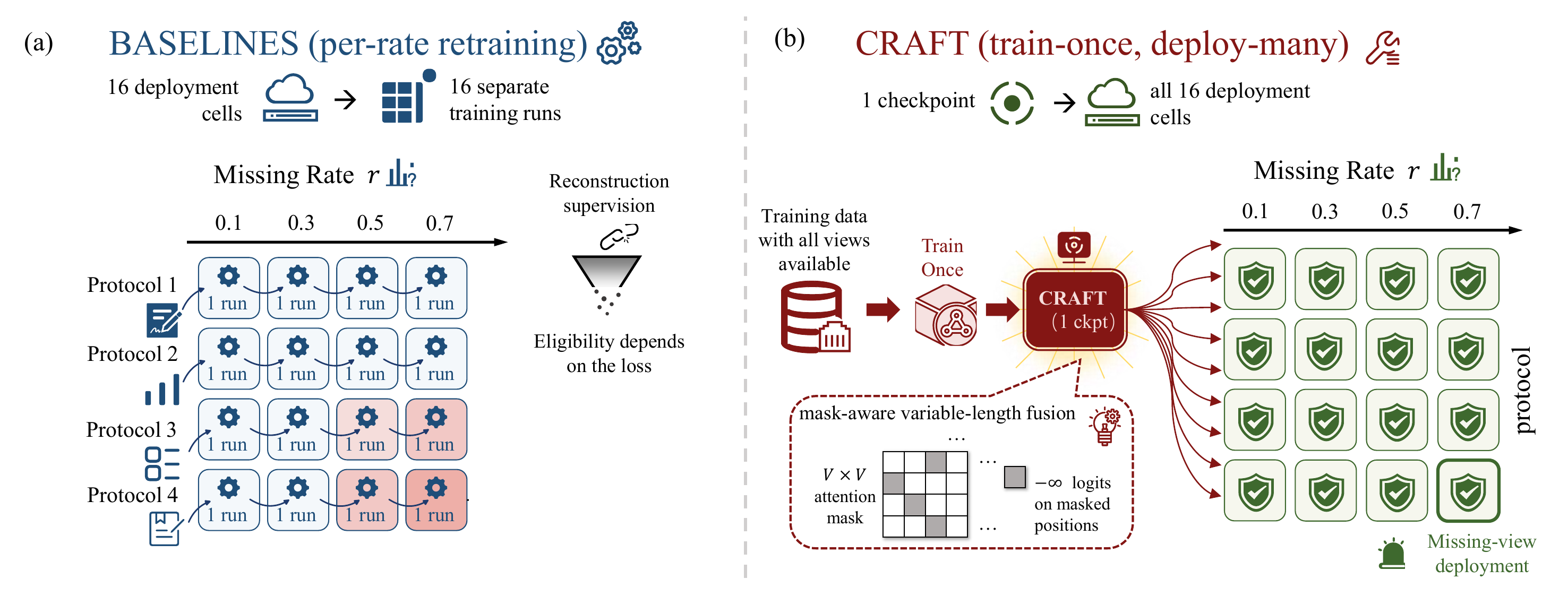}
  \caption{Per-configuration retraining vs.\ train-once deployment.
    \textbf{(a)}~Baselines retrain a separate model per
    $(\text{protocol}, r)$ configuration; reconstruction supervision
    depends on the observations required by the loss.
    \textbf{(b)}~CRAFT starts from training data with all views available
    and reuses one final checkpoint per dataset and seed across
    all $16$ deployment configurations (including $p_c \to 0$)
    via attention masking, replacing sixteen separate training runs.}
  \label{fig:paradigm}
\end{figure}


\section{Experimental Details}
\label{app:details}

\subsection{Datasets and Hardware}
\label{app:details:datasets}

Table~\ref{tab:datasets} summarizes the seven benchmarks used in
this work.

\begin{table}[h]
\centering
\caption{Dataset overview.}
\label{tab:datasets}
\small
\begin{tabular}{@{}l rrrr@{}}
\toprule
Dataset & $N$ & $V$ & $K$ & $d_{\text{total}}$ \\
\midrule
CUB           & 600     & 2 & 10 & 1{,}324 \\
MultiFashion  & 10{,}000 & 3 & 10 & 1{,}176 \\
UCI-digit     & 2{,}000 & 3 & 10 & 356 \\
Out-Scene     & 2{,}688 & 4 & 8  & 1{,}248 \\
HandWritten   & 2{,}000 & 6 & 10 & 649 \\
Caltech       & 2{,}386 & 6 & 20 & 3{,}766 \\
YTF-31        & 101{,}499 & 5 & 31 & 2{,}125 \\
\bottomrule
\end{tabular}
\end{table}

All experiments are conducted on a single NVIDIA RTX~3090 GPU
(24\,GB VRAM) with PyTorch~2.1 and CUDA~12.1. Timing records refer to this hardware. The information-matched
comparison reports five-seed means. In the historical deployment
tables, CRAFT and HSACC use five seeds (42--46), except YTF-31
(seed 42). DCG's fixed-seed repetition has one effective trajectory;
DVIMC on CUB and HandWritten is single-seed. The four Energy-DIMC
HandWritten entries at $r=0.7$ and the unannotated DCP/DCG entries
in the support comparison are single-seed. Other baseline aggregates
without an established independent-run count carry no five-seed
claim. Ablations use three seeds where stated. Standard deviations
are displayed only where supplied; these comparisons do not establish
statistical significance.

\subsection{CRAFT Hyperparameters}
\label{app:details:hyperparam}

Table~\ref{tab:hyperparams} lists the reference per-dataset configuration for deployment training.
All datasets use a single-layer Transformer with 4~heads,
dataset-specific MLP encoders and symmetric decoders, and the same
optimizer (AdamW, weight decay $10^{-4}$, batch size 256). For deployment, configuration selection is performed once using
training data with all views available, and the final checkpoint after the selected training
stages handles all missing patterns without further adjustment.
The CUB deployment results specify $\lambda_{\mathrm{repr}}=0.1$
and $\tau=0.5$ in Appendix~\ref{app:results:cub}. The retained
records do not establish the checkpoint-selection criterion for
every historical run.

\begin{table}[h]
\centering
\caption{Per-dataset hyperparameter configuration.
``Default'' denotes the shared configuration used by HW and UCI
without modification.}
\label{tab:hyperparams}
\small
\begin{tabular}{@{}l cccccc@{}}
\toprule
Parameter & Default & CUB & MF & OS & Cal & YTF \\
\midrule
encoder\_type & deep & deep & shallow & deep & shallow & deep \\
encoder\_hidden & 512 & 512 & 1024 & 512 & 512 & 1024 \\
embed\_dim $d$ & 256 & 128 & 256 & 256 & 256 & 512 \\
$\tau$ & 0.5 & 0.5 & 0.03 & 0.1 & 0.5 & 0.5 \\
$\lambda_{\mathrm{ent}}$ & 5.0 & 5.0 & 0.5 & 5.0 & 0.1 & 0.1 \\
dropout & 0.1 & 0.1 & 0.03 & 0.1 & 0.1 & 0.1 \\
noise\_std & 0.1 & 0.1 & 0.03 & 0.1 & 0.1 & 0.1 \\
Stage~1 epochs & 100 & 100 & 200 & 100 & 100 & 100 \\
Stage~2 epochs & 100 & 100 & 200 & 100 & 100 & 80 \\
Stage~1 lr & 5e-4 & 5e-4 & 5e-4 & 5e-4 & 5e-4 & 5e-4 \\
Stage~2 lr & 1e-5 & 1e-5 & 5e-5 & 1e-4 & 2e-4 & 1e-6 \\
\bottomrule
\end{tabular}
\end{table}

CUB uses $d{=}128$ because the small sample size ($n{=}600$) leads
to overfitting at $d{=}256$. Caltech and MultiFashion use shallow
encoders ($d_v {\to} d$) because their high-dimensional views cause
training instability with deeper encoders. YTF-31 uses a wider
architecture ($d{=}512$, hidden${=}1024$) to accommodate its larger
scale and higher class count ($K{=}31$). MFT resumes from the
Stage~2 checkpoint for an additional 100 epochs on HW and Out-Scene
or 300 epochs on MultiFashion, with learning rate $10^{-5}$.
These epochs are not included in the Stage~2 budgets in
Table~\ref{tab:hyperparams}.

The information-matched comparison trains CUB at Protocol~1
$r\in\{0.1,0.3,0.5,0.7\}$ and Protocol~4
$r\in\{0.1,0.3\}$, and MultiFashion at Protocol~1
$r\in\{0.1,0.3,0.5,0.7\}$ and Protocol~4
$r\in\{0.1,0.3,0.5\}$. Each condition uses its own training
run. Encoders and reconstruction targets use only observed entries;
pairwise consistency is averaged over eligible observed pairs and
is zero for singleton samples. Preprocessing and model selection
use the same training-information budget across compared methods.
The separate deployment hyperparameters above describe training that
starts with all views available.

\subsection{Baseline Reproduction}
\label{app:details:baselines}

This section consolidates reproduction details, protocol
inferences, and failure analyses for all baseline methods.

\subsubsection{DCG Configuration}
\label{app:details:dcg}

The official DCG repository contains no multi-view configuration
beyond two views. We extend the released two-view script to accept
all views, retaining all other hyperparameters at their official
values.

\subsubsection{View Selection Audit}
\label{app:details:views}

On HandWritten ($V{=}6$, $n{=}2{,}000$), the official DCG and
ProImp codebases load only $2$ of the $6$ views: DCG selects
views 2 and 3 (76d and 64d), discarding $78\%$ of the total
feature dimensions---a choice not documented in the corresponding
paper. ProImp's encoder pair is hardcoded to $2$ views by
construction. The DCG numbers reported in this paper instead use
our $6$-view extension (Appendix~\ref{app:details:dcg}), which
accesses all six views while retaining DCG's other hyperparameters
at their official values. ProImp and APADC are architecturally
limited to two views and are not extended in this work.

\subsubsection{DCP and DSIMVC Implementation}
\label{app:details:dcp}

\textbf{DCP.}
DCP shares its masking utility with COMPLETER (same research group
and codebase); its behavior matches Protocol~2
(\S\ref{sec:protocol:four}).

\textbf{DSIMVC.}
Generates masks by drawing a per-sample Bernoulli trial at rate $r$,
then deleting one randomly chosen view for each selected sample. A
post-processing step ensures at least one view per sample. This is
equivalent to Protocol~1 with Bernoulli rather than fixed-count
sampling; the difference is negligible for large $n$.

\textbf{Evaluation randomness.}
In DCG, the outer-loop variable intended to vary across repeated
runs is not connected to the random seed; a fixed seed produces
the same mask and training trajectory in every iteration, reducing
the effective number of independent trials to one. In ProImp, the
mask seed varies but the training seed is fixed. Repeating an identical trajectory does not add an independent
trial. DCG is therefore marked single-seed in the deployment
tables; the information-matched table has its own stated aggregation.

\subsubsection{BURG Protocol Inference}
\label{app:details:burg}

BURG~\citep{burg2025} does not specify its masking procedure.
We classify it as Protocol~2 by convention: BURG benchmarks
against COMPLETER, DCP, and APADC at identical nominal rates
($\text{MR} \in \{0.1, 0.3, 0.5, 0.7\}$) without discussing
protocol differences, and our audit
(\S\ref{app:protocol:audit}) verified by code inspection that
these three baselines realize Protocol~2 via a shared
entry-wise Bernoulli with protected-view masking codebase.
This suggests BURG inherits the same masking convention, though
the classification is not directly verifiable from the paper
alone.

\subsubsection{DVIMC Failure on Multi-View Datasets}
\label{app:details:dvimc}

DVIMC reports initialization \texttt{ValueError}s on MultiFashion
at $r=0.7$ (Protocol~4), CUB at $r\geq0.5$ (Protocol~4),
and HandWritten under Protocols~3--4 at $r=0.7$.
Protocols~3 and~4 retain at least one observed view per sample,
so these failures cannot be attributed to empty observed sets
under the stated masks. The retained records establish the error
category, but not a common root cause. COMPLETER's failed
MultiFashion/Protocol~4/$r=0.7$ entry records a data-loader
runtime error. These failed entries are excluded from numeric ranking.

Single-seed evaluation of DVIMC on CUB and HandWritten also
shows low and non-monotone ACC across successful configurations.
These observations concern the evaluated implementation and
hyperparameters, and do not establish a failure of all per-sample
architectures or the necessity of C2.

Under the identical $16$ configurations, the canonical CRAFT checkpoint
produces non-trivial accuracy on every condition with no crashes
(Table~\ref{tab:hw6v_full}).

\subsection{Tuning of Collapsed Baselines}
\label{app:details:tuning}

To test sensitivity to selected optimization settings,
we tuned DCG (lr$\in\{$1e-3, 5e-4, 1e-4$\}$, $\tau\in\{$0.5, 0.3,
0.1$\}$) and DCP (lr$\in\{$1e-3, 5e-4, 1e-4$\}$) under Protocol~4
on CUB ($V{=}2$, $r{=}0.5$, $p_c{=}0$). No DCG configuration
exceeded $20.17\%$ ACC and no DCP configuration exceeded
$36.83\%$ (random baseline $10\%$). These searches did not recover high accuracy within the tested
ranges; they do not exclude all optimization or implementation effects.

\subsection{Training Algorithm}
\label{app:details:algorithm}

\begin{algorithm}[H]
\caption{CRAFT training for default train-once deployment.
Pseudo-labels are refreshed at the start of each Stage~2 and MFT epoch.}
\label{alg:craft}
\begin{algorithmic}[1]
\REQUIRE complete samples $\{x^{(i)}\}_{i=1}^n$;
         encoder $\theta$, head $\phi$;
         epochs $E_1, E_2, E_{\mathrm{MFT}}$ ($E_{\mathrm{MFT}}=0$ if MFT is disabled);
         coefficients $\lambda_{\mathrm{repr}}, \beta, \gamma$
\FOR{each minibatch in $E_1$ epochs \COMMENT{Stage 1: representation learning}}
  \STATE Compute $z_v = E_v(x_v)$ for $v = 1{:}V$
  \STATE $h \leftarrow \mathrm{Transformer}([\mathrm{[CLS]}, z_1, \ldots, z_V])$ (no mask)
  \STATE Update $\theta$ on $\mathcal{L}_{\mathrm{stage1}}$ (Eq.~\ref{eq:stage1})
\ENDFOR
\FOR{each minibatch in $E_2$ epochs \COMMENT{Stage 2: cluster fine-tuning}}
  \STATE Compute $z_v = E_v(x_v)$ for $v = 1{:}V$
  \STATE $h \leftarrow \mathrm{Transformer}([\mathrm{[CLS]}, z_1, \ldots, z_V])$ (no mask)
  \STATE Update $\theta, \phi$ on $\mathcal{L}_{\mathrm{stage2}}$ with $\gamma=0$ (Eq.~\ref{eq:stage2})
\ENDFOR
\FOR{each minibatch in $E_{\mathrm{MFT}}$ additional epochs \COMMENT{Optional MFT continuation}}
  \STATE Compute $z_v = E_v(x_v)$; sample two MFT view-subsets $S_1,S_2$
  \STATE Compute $h_{S_1},h_{S_2}$ using their respective key-padding masks
  \STATE Update $\theta,\phi$ on $\mathcal{L}_{\mathrm{stage2}}$, including MFT consistency (Eq.~\ref{eq:stage2})
\ENDFOR
\STATE \textbf{Inference (nonempty observed-view subset):} forward pass with key-padding mask set from $O_i$.
\end{algorithmic}
\end{algorithm}

\subsection{Loss Component Definitions}
\label{app:details:losses}

For the default deployment-training recipe, the Stage~1 and Stage~2 losses
(Eqs.~\ref{eq:stage1}--\ref{eq:stage2})
use the following terms; entropy aggregates predictions across samples.

\paragraph{$\mathcal{L}_{\mathrm{recon}}$ (per-view reconstruction).}
$\displaystyle
\mathcal{L}_{\mathrm{recon}} \;=\; \tfrac{1}{n}\sum_{i=1}^{n}\sum_{v=1}^{V}
\bigl\|D_v(h^{(i)}) - x_v^{(i)}\bigr\|^2,
$
where $D_v$ is a view-specific decoder MLP.

\paragraph{$\mathcal{L}_{\mathrm{repr}}$ (cross-view consistency).}
$\displaystyle
\mathcal{L}_{\mathrm{repr}} \;=\; -\tfrac{1}{n V(V{-}1)}
\sum_{i}\sum_{v \neq v'}
\cos\!\bigl(P_\phi(z_v^{(i)}),\, \mathrm{sg}(z_{v'}^{(i)})\bigr),
$
with predictor MLP $P_\phi$ and stop-gradient
$\mathrm{sg}(\cdot)$~\citep{chen2021simsiam}.

\paragraph{$\mathcal{L}_{\mathrm{cluster}}$ (self-labeled cross-entropy).}
$\displaystyle
\mathcal{L}_{\mathrm{cluster}} \;=\; -\tfrac{1}{n}\sum_{i=1}^{n}
\log p_{\hat{y}^{(i)}}\!\bigl(h^{(i)}\bigr),
\quad
p_k(h) \;=\; \mathrm{softmax}_k\!\bigl(\tau\cos(h, c_k)\bigr),
$
where $\hat{y}^{(i)} = \arg\max_k p_k(h^{(i)})$ are pseudo-labels
refreshed each epoch and $\{c_k\}_{k=1}^K$ are unit-norm prototypes.

\paragraph{$\mathcal{L}_{\mathrm{ent}}$ (anti-collapse entropy).}
$\displaystyle
\mathcal{L}_{\mathrm{ent}} \;=\; \sum_{k=1}^{K} \bar{p}_k \log \bar{p}_k,
\qquad
\bar{p}_k \;=\; \tfrac{1}{n}\sum_{i=1}^{n} p_k\!\bigl(h^{(i)}\bigr).
$
Minimizing this term pushes the batch-averaged prediction toward
uniform over $K$ clusters, discouraging degenerate single-cluster
solutions.

\paragraph{$\mathcal{L}_{\mathrm{KL}}$ (MFT consistency).}
For two MFT-sampled view-subsets $S_1, S_2 \subseteq O_i$ of sample
$i$ (\S\ref{sec:method:mft}),
$\displaystyle
\mathcal{L}_{\mathrm{KL}} \;=\; \tfrac{1}{n}\sum_{i=1}^{n}
\mathrm{KL}\!\bigl(p(\cdot\mid h^{(i)}_{S_1}) \,\big\|\, p(\cdot\mid h^{(i)}_{S_2})\bigr),
$
with $\gamma = 0$ when MFT is disabled.

\subsection{PyTorch Inference Snippet}
\label{app:details:code}

The following snippet illustrates CRAFT's inference-time handling
of missing views. No special module is required; the standard
PyTorch \texttt{MultiheadAttention} accepts a key-padding mask
directly.

\begin{verbatim}
# tokens: [batch, 1+V, d]  (CLS + V view embeddings)
# missing: [batch, V] bool  (True = view absent)
pad = torch.zeros(B, 1+V, dtype=torch.bool)  # CLS never masked
pad[:, 1:] = missing
h = transformer_encoder(tokens, src_key_padding_mask=pad)
h_cls = h[:, 0]  # fused representation
pred = cosine_softmax(h_cls, centroids, tau)
\end{verbatim}


\section{Proofs}
\label{app:proofs}

\subsection{Proof of Proposition~\ref{prop:trainability}
            (Trainability Bound for $\mathcal{F}_{rec}$)}
\label{app:proofs:thm1a}

\begin{proof}[Proof of Proposition~\ref{prop:trainability}]
For a pairwise reconstruction branch, let
$\ell_i=\mathbf{1}\{|O_i|\geq2\}\tilde\ell_i$, with
$\|\nabla\tilde\ell_i\|\leq C_g$. The loss is normalized by
the total sample count, including ineligible samples with zero loss.
Then
\[
 \mathbb{E}\|\nabla\ell_i\|
 \leq C_g\,\mathbb{P}(|O_i|\geq2)=C_g q_2(P).
\]
For a uniformly sampled batch of size $B$, the averaged gradient
$g=B^{-1}\sum_{i=1}^{B}\nabla\ell_i$ satisfies
$\mathbb{E}\|g\|\leq B^{-1}\sum_i\mathbb{E}\|\nabla\ell_i\|
\leq C_g q_2(P)$ by the triangle inequality. A strict-complete
branch instead uses $\mathbf{1}\{|O_i|=V\}$, giving $C_g p_c$.
For a loss associated with one specific view pair, use that pair's
co-observation probability; $q_2$ remains an upper envelope.
\end{proof}

\paragraph{Normalization and active batches.}
The same support can yield a different bound when gradients are
averaged over eligible samples only. For a strict-complete branch,
let $B_c\sim\mathrm{Binomial}(B,p_c)$ under iid mask sampling,
and define
\[
g_t=\begin{cases}
 B_c^{-1}\sum_{i\in C_{\mathrm{batch}}}\nabla\tilde\ell_i,
 & B_c\geq1,\\
 0,&B_c=0.
\end{cases}
\]
Here $C_{\mathrm{batch}}$ is the set of complete samples in the
batch. Then $\mathbb{E}\|g_t\|\leq
C_g[1-(1-p_c)^B]$, and the expected active-batch size is
$\mathbb{E}[B_c\mid B_c\geq1]=Bp_c/[1-(1-p_c)^B]$ for $p_c>0$.
For $B=256$, the active-batch probability rises from $10\%$ to
$90\%$ between $p_c\approx0.04\%$ and $0.90\%$.
These are batch-availability levels, not accuracy thresholds.
Uniform resampling of a fixed masked dataset uses its empirical
complete fraction; eligibility-stratified sampling changes this law.
Neither this counting argument nor a standard convergence upper
bound establishes a necessary $1/p_c$ iteration lower bound.

\paragraph{Remark (fixed-topology failure, ProImp case).}
The \texttt{cross\_rec()} module of ProImp constructs tensors of
shape $|C|\times K$ and $|C|\times|C|$, where $C$ is the
complete-sample set. Under Protocol~4 with $V=2$ and $r\geq0.5$,
$|C|=0$. The recorded fallback returns incompatible tensor shapes,
triggering \texttt{RuntimeError: The size of tensor a (1) must
match the size of tensor b (600)}. This is an implementation
failure at empty complete-sample support.

\paragraph{Remark (CRAFT forward computation).}
For fixed parameters, sample $i$'s representation uses its observed
features and the key-padding mask. The fixed-slot implementation
retains $V+1$ tokens, including CLS, and masks missing keys;
its matrix dimensions do not shrink with $|O_i|$.
The forward map remains defined for any nonempty $O_i$, including
deployment configurations with $p_c=0$. This property is separate
from the training-information conditions of each experiment.

\subsection{Capability Bound: Theorem Statement and Proof}
\label{app:proofs:thm1b}

This subsection states and proves the capability bound referenced
throughout the main text (\S\ref{sec:theory},
\S\ref{sec:exp:main}, \S\ref{sec:discussion}). The bound limits
the cluster-relevant information any clustering function can
extract from a given observed-view subset, regardless of
training dynamics or architecture, and applies to all methods
including those outside $\mathcal{F}_{rec}$ (e.g., distributional
methods such as Energy-DIMC).

We prove Theorem~\ref{thm:capability} (statement in
\S\ref{sec:theory}). The upper inequality requires conditional independence and
finite information quantities; without that assumption, the lower
inequality and data-processing upper bound still hold.

\begin{proof}[Proof of Theorem~\ref{thm:capability}]

\textbf{Upper bound.}
Under Assumption~\ref{ass:cond_ind} (conditional independence
of views given $Y$), the joint entropy decomposes as
$H(\{X_v\}_{v \in O_i} \mid Y) =
\sum_{v \in O_i} H(X_v \mid Y)$. By the subadditivity of
entropy, $H(\{X_v\}_{v \in O_i}) \leq
\sum_{v \in O_i} H(X_v)$. Therefore,
\begin{align}
  I(\{X_v\}_{v \in O_i}; Y)
  &= H(\{X_v\}_{v \in O_i})
     - H(\{X_v\}_{v \in O_i} \mid Y) \nonumber \\
  &\leq \sum_{v \in O_i} H(X_v)
     - \sum_{v \in O_i} H(X_v \mid Y) \nonumber \\
  &= \sum_{v \in O_i} I(X_v; Y).
\end{align}

\textbf{Lower bound.}
For any $u \in O_i$, $X_u$ is a component of
$\{X_v\}_{v \in O_i}$. Adding more variables cannot
decrease mutual information:
$I(\{X_v\}_{v \in O_i}; Y) \geq I(X_u; Y)$. Taking the
maximum over $u \in O_i$ gives the lower bound.

\textbf{Application to fused representations.}
Since the fused representation
$\mathbf{h}_{\mathrm{cls}}^{(i)}$ is a deterministic function of
$\{X_v\}_{v \in O_i}$, the data processing
inequality gives $I(\mathbf{h}_{\mathrm{cls}}^{(i)}; Y) \leq
I(\{X_v\}_{v \in O_i}; Y)$. The lower bound on the observed set does not transfer to the
learned representation: a constant encoder can discard all
cluster information. When $|O_i|=1$, the observed-set information
is that of the single retained view; architecture and optimization
still affect how much of it is preserved. These statements apply
to fixed encoders on the evaluation distribution, with any
training information treated as fixed.
\end{proof}

\input{theorem2_proof}


\subsection{Cluster-Information Bound: $Y$ is the Statistical Bridge}
\label{app:proofs:bridge}

\begin{corollary}[Recovery does not increase available information
at inference]
\label{cor:recovery_redundancy}
Let $\hat{X}_v = g_\phi(\{X_u\}_{u \in O_i})$ be the output
of a recovery module applied to the observed views at inference.
Then
$I(\{X_u\}_{u \in O_i},\, \hat{X}_v;\; Y)
= I(\{X_u\}_{u \in O_i};\; Y)$.
\end{corollary}
\begin{proof}
Direct from the data processing inequality: $\hat{X}_v$ is a
deterministic function of $\{X_u\}_{u \in O_i}$, so
appending it to the conditioning set cannot increase information
about $Y$.
\end{proof}

This concerns inference-time information content. During training,
recovery modules provide a useful gradient signal through cross-view
reconstruction. The following inequality concerns distinct views.

\begin{proposition}[Cross-view predictive information]
\label{prop:bridge}
Under Assumption~\ref{ass:cond_ind}, for any fixed encoder
$f_\theta$ and any two distinct views $u\ne v$:
\begin{equation}
  I(f_\theta(X_u);\; X_v)
  \;\leq\;
  I(f_\theta(X_u);\; Y).
  \label{eq:recon_bound}
\end{equation}
The inequality is pointwise in the encoder and does not assert
that both information quantities increase together during optimization.
\end{proposition}

\begin{proof}
Since $f_\theta(X_u)$ is a deterministic function of $X_u$, and
$X_u \perp X_v \mid Y$ by Assumption~\ref{ass:cond_ind},
$f_\theta(X_u) \perp X_v \mid Y$, giving
$I(f_\theta(X_u); X_v \mid Y) = 0$. Applying the chain rule:
\begin{align}
  I(f_\theta(X_u);\; X_v)
  &= I(f_\theta(X_u);\; X_v, Y)
     - I(f_\theta(X_u);\; Y \mid X_v) \nonumber \\
  &\leq I(f_\theta(X_u);\; X_v, Y) \nonumber \\
  &= I(f_\theta(X_u);\; Y)
     + \underbrace{I(f_\theta(X_u);\; X_v \mid Y)}_{= 0}. \qedhere
\end{align}
\end{proof}

\textbf{Interpretation.}
Eq.~\eqref{eq:recon_bound} says that the cross-view predictive
information is bounded above by the cluster-relevant information.
Under conditional independence, $Y$ is the sole statistical bridge
between $X_u$ and $X_v$: any feature of $X_u$ that is predictive
of $X_v$ must pass through $Y$. The inequality concerns an encoder of $X_u$ alone. CRAFT's
fusion with all views available also observes the target view $X_v$, so
this argument does not identify its reconstruction loss with
cluster-information maximization.

\subsection{Attention Reweighting and Exact Exclusion}
\label{app:proofs:exclusion}

\begin{proposition}[Attention Leakage Bound for Placeholder Tokens]
\label{prop:attn_bound}
Consider attention over $V$ view tokens, with $k<V$ absent
positions and retained-token logits bounded below by $s_{\min}$.
When absent views are replaced by learnable placeholder tokens
with logits $\leq\varepsilon$, their total attention weight satisfies
\[
  \sum_{j \in \mathrm{mask}} \alpha_j
  \;\leq\;
  \frac{k}{V - k} \cdot \exp(\varepsilon - s_{\min}).
\]
\end{proposition}

\begin{proof}
The mask set has $k$ elements with scaled dot product
$\leq \varepsilon$:
$\sum_{j \in \mathrm{mask}} \exp(\mathbf{q}^\top \mathbf{k}_j /
\sqrt{d}) \leq k \cdot \exp(\varepsilon)$. The denominator
$Z \geq (V - k) \cdot \exp(s_{\min})$. Dividing yields the
claim.
\end{proof}

\textbf{Corollary (Exact exclusion under attention masking).}
Under canonical CRAFT, $\varepsilon = -\infty$ (attention masking
sets the logit to $-\infty$ before the softmax) and the bound
yields $\sum_{j \in \mathrm{mask}} \alpha_j = 0$ exactly. This
guarantees that absent views contribute zero information to the
CLS representation, giving exact exclusion as required by C2.
This forward property alone does not specify the accuracy response
to removing an informative view.

\textbf{Empirical verification.}
On HandWritten ($V{=}6$) with $k{=}1$, the learnable-placeholder
variant allocates $9.4\%$ of CLS attention to the placeholder
versus the $16.7\%$ uniform share, confirming that the trained
model systematically aligns real-token keys more closely with the
CLS query but does not achieve exact exclusion---only canonical
attention masking does.

\subsection{MFT as Weighted View-Subset Ensemble}
\label{app:proofs:mft}

\begin{proposition}[View Dropout Ensemble]
\label{prop:view_dropout}
In MFT, each sample draws $k \sim \mathrm{Uniform}(\{0, \ldots,
V{-}2\})$ views to exclude. The expected MFT loss equals a
weighted sum over view subsets:
\[
  \mathbb{E}_{\mathbf{m}}[\mathcal{L}(f_\theta(\mathbf{x},
  \mathbf{m}))]
  =
  \sum_{S : |S| \geq 2} P(S) \cdot
  \mathcal{L}(f_\theta(\mathbf{x}_S)),
\]
where $P(S) = 1/((V{-}1) \binom{V}{V - |S|})$.
\end{proposition}

\begin{proof}
Each mask determines a unique retained subset. The probability of
a specific subset $S$ with $|S| = V - k$ is
$\frac{1}{V-1} \cdot \frac{1}{\binom{V}{k}}$. Normalization is
verified by
$\sum_{k=0}^{V-2} \binom{V}{k} / ((V-1) \binom{V}{k}) = 1$.
\end{proof}

\textbf{Interpretation.}
(1)~The full-view subset ($k{=}0$) appears with probability
$1/(V{-}1)$, retaining complete inputs in the training objective.
(2)~At $V{=}6$, $57$ distinct subsets participate in the ensemble,
providing strong regularization; at $V{=}2$, only $1$ subset (full
data) participates, so MFT performs no view deletion on CUB.
(3)~This identity expresses a weighted training objective; inference
uses one forward pass for the observed subset, and the identity
provides no accuracy or ensemble-improvement guarantee.


\section{Protocol Analysis}
\label{app:protocol}

\subsection{Closed-Form Protocol Formulas}
\label{app:protocol:formulas}

For each protocol introduced in \S\ref{sec:protocol:four}, we
state the closed-form expressions for the realized entry-wise
missing rate $\hat{r}$ and the complete-sample proportion $p_c$ in
the limit $n \to \infty$. Throughout, $r \in [0, 1]$ denotes the
nominal missing rate as reported by each method.

\textbf{Protocol~1 (sample-level; delete one view per sample).}
A fraction $r$ of samples are designated incomplete; for each
incomplete sample, exactly one view is deleted uniformly at random.
Then:
\[
  \hat{r}_{\mathrm{P1}} = \frac{r}{V},
  \qquad
  p_c^{\mathrm{P1}} = 1 - r.
\]

\textbf{Protocol~2 (protected view; entry-wise at $r/V$).}
For each sample, exactly one view is randomly designated as
``protected'' (always observed); the remaining $V{-}1$ views are
masked by independent Bernoulli$(r/(V{-}1))$ trials, calibrated to
realize the nominal entry-wise rate $r/V$ in expectation. Then:
\[
  \hat{r}_{\mathrm{P2}} = \frac{r}{V},
  \qquad
  p_c^{\mathrm{P2}} = \left(1 - \frac{r}{V-1}\right)^{V-1}.
\]

\textbf{Protocol~3 (entry-wise Bernoulli at $r$; one-hot
guarantee).}
Each entry $M_v^{(i)}$ is masked by an independent Bernoulli$(r)$
trial; any sample that loses all $V$ views has one randomly chosen
view restored (the one-hot guarantee). Then:
\[
  \hat{r}_{\mathrm{P3}} = r - \frac{r^V}{V},
  \qquad
  p_c^{\mathrm{P3}} = (1-r)^V.
\]

\textbf{Protocol~4 (protected view; fixed deletion budget at $r$).}
For each sample, one protected view is fixed. A total of
$D=\min(\operatorname{round}(rnV),n(V-1))$ entries are deleted uniformly
without replacement from the $N=n(V-1)$ non-protected entries.
Here $\operatorname{round}$ rounds to the nearest integer, with ties to even.
Thus $\hat r=D/(nV)$. The probability that a given sample stays
complete is $\binom{N-(V-1)}{D}/\binom{N}{D}$, interpreted as
zero when too few entries remain. For fixed $V$, its large-$n$
limit gives
\[
 \hat r_{\mathrm{P4}}=\min\!\left(r,\frac{V-1}{V}\right),
 \qquad
 p_c^{\mathrm{P4}}=
 \left[\max\!\left(0,1-\frac{Vr}{V-1}\right)\right]^{V-1}.
\]
The finite-sample deletion budget couples entries; independent
Bernoulli trials give the same limit but a different mask law.

\textbf{Numerical separation.}
At low $r$, the four formulas for $p_c$ are nearly indistinguishable
($p_c \to 1$ as $r \to 0$). At moderate-to-high $r$ and $V \geq 3$,
they diverge by orders of magnitude---e.g., on $V{=}6$, $r{=}0.7$:
$p_c^{\mathrm{P1}} = 30\%$ vs.\ $p_c^{\mathrm{P4}} \approx
1.0 \times 10^{-4}$, a $\sim$$3{,}000{\times}$ gap at the same
nominal rate. Numerical instantiations appear in
Table~\ref{tab:hidden_vars}.

\subsection{Monte Carlo Verification of Protocol Formulas}
\label{app:protocol:mc}

We verify the closed-form expressions for $\hat{r}$ and $p_c$
through Monte Carlo simulation ($n{=}2000$, $V{=}6$, 100
independent trials per combination). Empirical means match
analytical predictions within $0.5\%$ in all cases.

\textbf{Protected-view ceiling.}
The protected view imposes an upper bound of $(V{-}1)/V$ on
$\hat{r}$ under Protocol~4. When $V{=}2$, $r{=}0.5$ and $r{=}0.7$
generate the same saturated mask distribution (and identical
masks when the protected-view draw is shared). When $V{=}3$, the ceiling is
$2/3 \approx 0.667$, so $r{=}0.7$ cannot be fully realized.
Thus these two nominal rates are not distinct observation
conditions for two-view datasets under Protocol~4.

\textbf{DSIMVC protocol equivalence.}
DSIMVC uses per-sample Bernoulli trials rather than the
fixed-count sampling of Protocol~1. By the law of large numbers,
the fraction of incomplete samples converges to $r$ as $n \to
\infty$; for $n{=}2000$, the standard deviation is
$\sqrt{r(1-r)/n} < 1.2\%$. The two procedures share these limiting statistics, while
fixed-count and Bernoulli sampling have different finite-sample
count variability.

\subsection{Hidden-Variable Numerical Profiles}
\label{app:protocol:profiles}

Table~\ref{tab:hidden_vars} instantiates the closed-form
expressions for two representative configurations.

\begin{table}[h]
\centering
\caption{Hidden variables ($\hat{r}$ and $p_c$) under four
protocols, computed from the limiting formulas in
\S\ref{app:protocol:formulas}.}
\label{tab:hidden_vars}
\small
\begin{tabular}{@{}l cc cc@{}}
\toprule
& \multicolumn{2}{c}{$V{=}2,\; r{=}0.5$}
& \multicolumn{2}{c}{$V{=}6,\; r{=}0.7$} \\
\cmidrule(lr){2-3} \cmidrule(lr){4-5}
Protocol & $\hat{r}$ & $p_c$ & $\hat{r}$ & $p_c$ \\
\midrule
1 & 25.0\% & 50.0\%  & 11.7\% & 30.0\% \\
2 & 25.0\% & 50.0\%  & 11.7\% & 47.0\% \\
3 & 37.5\% & 25.0\%  & 68.0\% & 0.073\% \\
4 & 50.0\% & 0\%     & 70.0\% & 0.010\% \\
\bottomrule
\end{tabular}
\end{table}

The right-hand block ($V{=}6$, $r{=}0.7$) exhibits a
$\sim$$4500{\times}$ divergence in $p_c$ between Protocol~2
($47\%$) and Protocol~4 ($0.010\%$) at the same nominal missing
rate---a difference in complete-sample support at the same nominal rate. Even the
intermediate comparison Protocol~2 vs.\ Protocol~3 ($47\%$ vs.\
$0.073\%$) shows a $640{\times}$ divergence.

\subsection{Protocol Mappings in Nine Implementations}
\label{app:protocol:audit}

To determine which protocol each method actually uses, we
inspected the official code repositories of nine IMVC methods
published between 2021 and 2025. Table~\ref{tab:audit} records the
masking mechanism observed in each implementation and identifies
the closest matching protocol from
\S\ref{app:protocol:formulas}.

\begin{table}[t]
\centering
\caption{Masking mechanism in nine IMVC implementations, mapped
to the closest protocol from \S\ref{app:protocol:formulas}.
$\ddagger$~shared masking code; $\dagger$~per-sample Bernoulli
vs.\ fixed-count (\S\ref{app:protocol:mc}); $+$~Protocol~2 with
a deterministic complete-sample floor.}
\label{tab:audit}
\small
\setlength{\tabcolsep}{4pt}
\begin{tabular}{@{}>{\raggedright\arraybackslash}p{0.27\linewidth}l >{\raggedright\arraybackslash}p{0.46\linewidth}c@{}}
\toprule
Method & Venue & Masking Mechanism & Proto. \\
\midrule
COMPLETER$^\ddagger$~\citep{lin2021completer}
  & CVPR'21
  & $r/V$ entry-wise + protected view + calibration
  & 2 \\
DCP$^\ddagger$~\citep{lin2022dcp}
  & TPAMI'22
  & $r/V$ entry-wise + protected view (shared codebase)
  & 2 \\
DSIMVC
  & ICML'22
  & per-sample Bernoulli, post-fix $|O_i| \geq 1$
  & 1$^\dagger$ \\
ProImp$^\ddagger$~\citep{li2023proimp}
  & IJCAI'23
  & $r/V$ entry-wise + protected view
  & 2 \\
APADC$^\ddagger$~\citep{xu2023apadc}
  & TIP'23
  & $r/V$ entry-wise + protected view
  & 2 \\
DVIMC~\citep{chen2025dvimc}
  & MM'25
  & sample-level; delete $1{\sim}V{-}1$ views per incomplete sample
  & 1 \\
DCG$^\ddagger$~\citep{zhang2025dcg}
  & AAAI'25
  & $r/V$ entry-wise + protected view
  & 2 \\
Energy-DIMC~\citep{wang2025energydimc}
  & MM'25
  & $\mu/V$ entry-wise + per-sample protected view
  & 2 \\
HSACC~\citep{hsacc2025}
  & NeurIPS'25
  & $(1{-}r)n$ complete samples $+$ entry-wise masking on rest
  & 2$^+$ \\
\bottomrule
\end{tabular}
\end{table}

The realized protocol distribution is concentrated:
\textbf{seven of nine methods} sit at Protocol~2, and
\textbf{the remaining two} (DSIMVC, DVIMC) use sample-level deletion.
DVIMC deletes a variable number of views, so its mapping is a
mechanism-level analogy, not an exact match to Protocol~1's
effective-rate formula. The
stricter regimes (Protocols~3 and~4) are not represented among
these nine audited implementations, motivating their inclusion
in the deployment and incomplete-training evaluations.


\section{Complete Experimental Results}
\label{app:results}

\subsection{CUB ($V{=}2$)}
\label{app:results:cub}

CRAFT uses one canonical checkpoint per seed
($\lambda_{\mathrm{repr}}{=}0.1$, $\tau{=}0.5$, no MFT);
baselines retrain per configuration. Under Protocol~4 with $V{=}2$ the
protected view caps $\hat{r}$ at $0.5$, so Protocol~4 $r{=}0.7$
has the same saturated mask law as $r{=}0.5$ (\S\ref{app:protocol:mc}).

\begin{table}[h]
\centering
\caption{ACC (\%) on CUB ($V{=}2$, $K{=}10$).
$^\dagger$Protocol~4 $r{=}0.7$ has the same mask law as $r{=}0.5$.
$^*$BURG cited from its paper; --- = not reported. DCG has one
effective seed; independent-run counts follow Appendix~\ref{app:details:datasets}.}
\label{tab:cub_full}
\small
\begin{tabular}{@{}ll rrrr@{}}
\toprule
Method & $r$ & Proto.~1 & Proto.~2 & Proto.~3 & Proto.~4 \\
\midrule
\multirow{4}{*}{DCG}
  & 0.1 & 71.50 & 74.11 & 70.50 & 70.28 \\
  & 0.3 & 72.56 & 70.33 & 67.78 & 65.39 \\
  & 0.5 & 73.33 & 69.11 & 60.17 & 15.61 \\
  & 0.7 & 63.50 & 59.16 & 55.95 & 17.72 \\
\midrule
\multirow{4}{*}{COMPLETER}
  & 0.1 & 51.33 & 55.80 & 63.23 & 61.23 \\
  & 0.3 & 61.60 & 62.43 & 52.10 & 18.27 \\
  & 0.5 & 49.83 & 53.67 & 19.43 & 36.43 \\
  & 0.7 & 18.90 & 18.50 & 21.13 & 36.43 \\
\midrule
\multirow{4}{*}{DCP}
  & 0.1 & 65.00 & 59.83 & 62.33 & 64.17 \\
  & 0.3 & 57.50 & 62.67 & 44.83 & 29.67 \\
  & 0.5 & 51.33 & 51.83 & 28.17 & 36.83 \\
  & 0.7 & 26.00 & 27.00 & 34.33 & 36.83 \\
\midrule
\multirow{4}{*}{Energy-DIMC}
  & 0.1 & 75.20 & 68.55 & 71.48 & 71.29 \\
  & 0.3 & 67.77 & 74.61 & 65.62 & 53.52 \\
  & 0.5 & 58.20 & 65.23 & 47.85 & 42.38 \\
  & 0.7 & 51.56 & 59.77 & 45.51 & 42.38 \\
\midrule
\multirow{4}{*}{BURG$^*$}
  & 0.1 & --- & 71.50 & --- & --- \\
  & 0.3 & --- & 58.00 & --- & --- \\
  & 0.5 & --- & 57.33 & --- & --- \\
  & 0.7 & --- & 53.83 & --- & --- \\
\midrule
\multirow{4}{*}{HSACC}
  & 0.1 & 63.13\scriptsize{$\pm$1.2} & 63.63\scriptsize{$\pm$2.2} & 62.23\scriptsize{$\pm$2.4} & 59.50\scriptsize{$\pm$2.3} \\
  & 0.3 & 54.00\scriptsize{$\pm$3.2} & 57.20\scriptsize{$\pm$4.1} & 45.03\scriptsize{$\pm$4.0} & 33.13\scriptsize{$\pm$1.9} \\
  & 0.5 & 43.40\scriptsize{$\pm$3.2} & 45.63\scriptsize{$\pm$3.1} & 30.70\scriptsize{$\pm$0.6} & 38.63\scriptsize{$\pm$1.5} \\
  & 0.7 & 29.67\scriptsize{$\pm$0.5} & 30.67\scriptsize{$\pm$1.4} & 35.00\scriptsize{$\pm$0.6} & 38.63\scriptsize{$\pm$1.5} \\
\midrule
\multirow{4}{*}{CRAFT}
  & 0.1 & 82.09\scriptsize{$\pm$1.0} & 82.35\scriptsize{$\pm$0.9} & 80.92\scriptsize{$\pm$1.2} & 80.89\scriptsize{$\pm$1.2} \\
  & 0.3 & 79.19\scriptsize{$\pm$1.1} & 79.44\scriptsize{$\pm$1.3} & 76.27\scriptsize{$\pm$1.6} & 74.15\scriptsize{$\pm$2.2} \\
  & 0.5 & 76.04\scriptsize{$\pm$2.2} & 76.41\scriptsize{$\pm$1.9} & 72.29\scriptsize{$\pm$2.4} & 68.99\scriptsize{$\pm$3.8} \\
  & 0.7 & 73.61\scriptsize{$\pm$2.7} & 73.23\scriptsize{$\pm$2.7} & 69.81\scriptsize{$\pm$2.9} & 68.99\scriptsize{$\pm$3.8}$^\dagger$ \\
\bottomrule
\end{tabular}
\end{table}

\subsection{HandWritten ($V{=}6$)}
\label{app:results:hw}


Per-(protocol, rate) accuracy on HandWritten under canonical CRAFT
(Stage~1 + Stage~2 + MFT, per \S\ref{sec:exp:setup}); baselines
retrain per configuration, while CRAFT starts from training data with all views
available and reuses one final checkpoint per seed.

\begin{table}[h]
\centering
\caption{ACC (\%) on HandWritten ($V{=}6$, $K{=}10$). CRAFT and HSACC:
5-seed mean $\pm$ std. Energy-DIMC at $r=0.7$ and DCP at
$r=0.7$ under Protocols~1/4 are single-seed. Other baseline
entries without supplied std have no established independent-run
count. $^\dagger$BURG cited from its paper; --- = not reported.}
\label{tab:hw6v_full}
\small
\begin{tabular}{@{}ll rrrr@{}}
\toprule
Method & $r$ & Proto.~1 & Proto.~2 & Proto.~3 & Proto.~4 \\
\midrule
\multirow{4}{*}{DCP}
  & 0.1 & 82.35 & 82.30 & 79.95 & 82.10 \\
  & 0.3 & 83.90 & 84.00 & 22.45 & 24.60 \\
  & 0.5 & 80.00 & 93.65 & 18.15 & 18.15 \\
  & 0.7 & 71.55 & 70.10 & 16.40 & 16.25 \\
\midrule
\multirow{4}{*}{Energy-DIMC}
  & 0.1 & 96.65 & 96.26 & 96.43 & 96.09 \\
  & 0.3 & 96.37 & 96.54 & 95.20 & 94.87 \\
  & 0.5 & 96.09 & 96.09 & 92.24 & 92.75 \\
  & 0.7 & 96.37 & 95.65 & 83.98 & 83.26 \\
\midrule
\multirow{4}{*}{BURG$^\dagger$}
  & 0.1 & --- & 95.70 & --- & --- \\
  & 0.3 & --- & 95.30 & --- & --- \\
  & 0.5 & --- & 94.30 & --- & --- \\
  & 0.7 & --- & 79.75 & --- & --- \\
\midrule
\multirow{4}{*}{HSACC}
  & 0.1 & 76.53\scriptsize{$\pm$3.7} & 80.48\scriptsize{$\pm$8.6} & 87.89\scriptsize{$\pm$7.4} & 89.39\scriptsize{$\pm$1.3} \\
  & 0.3 & 87.91\scriptsize{$\pm$4.2} & 85.40\scriptsize{$\pm$4.1} & 14.92\scriptsize{$\pm$0.2} & 15.05\scriptsize{$\pm$0.2} \\
  & 0.5 & 87.10\scriptsize{$\pm$5.0} & 84.10\scriptsize{$\pm$5.8} & 14.08\scriptsize{$\pm$0.5} & 14.24\scriptsize{$\pm$0.2} \\
  & 0.7 & 80.56\scriptsize{$\pm$1.7} & 83.48\scriptsize{$\pm$1.3} & 13.80\scriptsize{$\pm$0.1} & 13.52\scriptsize{$\pm$0.6} \\
\midrule
\multirow{4}{*}{CRAFT}
  & 0.1 & 97.60\scriptsize{$\pm$0.8} & 97.61\scriptsize{$\pm$0.8} & 97.49\scriptsize{$\pm$0.9} & 97.53\scriptsize{$\pm$0.8} \\
  & 0.3 & 97.57\scriptsize{$\pm$0.8} & 97.59\scriptsize{$\pm$0.8} & 96.51\scriptsize{$\pm$0.9} & 96.82\scriptsize{$\pm$0.8} \\
  & 0.5 & 97.57\scriptsize{$\pm$0.8} & 97.53\scriptsize{$\pm$0.8} & 92.92\scriptsize{$\pm$1.3} & 93.69\scriptsize{$\pm$1.1} \\
  & 0.7 & 97.58\scriptsize{$\pm$0.9} & 97.49\scriptsize{$\pm$0.9} & 86.03\scriptsize{$\pm$2.1} & 85.21\scriptsize{$\pm$2.4} \\
\bottomrule
\end{tabular}
\end{table}


\subsection{MultiFashion ($V{=}3$)}
\label{app:results:mf}

Table~\ref{tab:mf_full} reports per-(protocol, $r$) accuracy on
MultiFashion; baselines retrain per configuration, while CRAFT starts from
training data with all views available, with MFT enabled at $V{\geq}3$.

\begin{table}[h]
\centering
\caption{ACC (\%) on MultiFashion ($V{=}3$, $K{=}10$). CRAFT uses
attention masking with MFT; CRAFT and HSACC report 5-seed means.
Other baseline run counts follow Appendix~\ref{app:details:datasets};
--- = runtime failure.}
\label{tab:mf_full}
\small
\begin{tabular}{@{}ll rrrr@{}}
\toprule
Method & $r$ & Proto.~1 & Proto.~2 & Proto.~3 & Proto.~4 \\
\midrule
\multirow{4}{*}{DCP}
  & 0.1 & 86.04 & 74.92 & 72.57 & 75.24 \\
  & 0.3 & 72.75 & 75.21 & 66.18 & 67.21 \\
  & 0.5 & 72.26 & 80.46 & 63.78 & 69.36 \\
  & 0.7 & 68.95 & 78.73 & 33.21 & 25.11 \\
\midrule
\multirow{4}{*}{Energy-DIMC}
  & 0.1 & 93.30 & 93.88 & 91.20 & 91.97 \\
  & 0.3 & 91.64 & 91.27 & 87.75 & 88.53 \\
  & 0.5 & 90.19 & 90.38 & 83.09 & 79.77 \\
  & 0.7 & 89.90 & 89.61 & 72.60 & 29.14 \\
\midrule
\multirow{4}{*}{HSACC}
  & 0.1 & 97.46\scriptsize{$\pm$0.8} & 98.00\scriptsize{$\pm$0.2}
        & 96.02\scriptsize{$\pm$0.2} & 94.34\scriptsize{$\pm$3.5} \\
  & 0.3 & 96.40\scriptsize{$\pm$0.3} & 96.29\scriptsize{$\pm$0.1}
        & 85.17\scriptsize{$\pm$3.8} & 82.50\scriptsize{$\pm$2.5} \\
  & 0.5 & 91.67\scriptsize{$\pm$3.9} & 94.33\scriptsize{$\pm$0.4}
        & 65.48\scriptsize{$\pm$2.8} & 49.80\scriptsize{$\pm$1.6} \\
  & 0.7 & 83.36\scriptsize{$\pm$3.0} & 88.03\scriptsize{$\pm$3.6}
        & 17.62\scriptsize{$\pm$1.3} & 21.61\scriptsize{$\pm$0.4} \\
\midrule
\multirow{4}{*}{DVIMC}
  & 0.1 & 83.86 & 86.51 & 79.95 & 86.65 \\
  & 0.3 & 86.62 & 88.67 & 83.51 & 84.09 \\
  & 0.5 & 88.82 & 85.81 & 80.67 & 80.43 \\
  & 0.7 & 87.21 & 84.87 & 74.42 & --- \\
\midrule
\multirow{4}{*}{CRAFT}
  & 0.1 & 93.21 & 93.20 & 92.73 & 92.73 \\
  & 0.3 & 92.82 & 92.70 & 90.66 & 90.83 \\
  & 0.5 & 92.44 & 92.11 & 88.67 & 88.12 \\
  & 0.7 & 91.97 & 91.38 & 86.64 & 85.41 \\
\bottomrule
\end{tabular}
\end{table}

\subsection{UCI-Digit, Out-Scene, and Caltech}
\label{app:results:supp}

A single CRAFT checkpoint per dataset and seed is evaluated across all
$16$ ($\text{protocol}, r$) configurations, mirroring the train-once
protocol in \S\ref{sec:exp:setup}; Table~\ref{tab:supp_full}
reports the results. Baselines were not re-run on these three
supplementary benchmarks.

\begin{table}[h]
\centering
\caption{CRAFT ACC (\%) on UCI-Digit ($V{=}3$, $K{=}10$),
Out-Scene ($V{=}4$, $K{=}8$), and Caltech ($V{=}6$, $K{=}20$).
5-seed mean $\pm$ std.}
\label{tab:supp_full}
\small
\begin{tabular}{@{}ll rrrr@{}}
\toprule
Dataset & $r$ & Proto.~1 & Proto.~2 & Proto.~3 & Proto.~4 \\
\midrule
\multirow{4}{*}{UCI-Digit}
  & 0.1 & 90.94\scriptsize{$\pm$0.6} & 90.87\scriptsize{$\pm$0.5}
        & 90.22\scriptsize{$\pm$0.4} & 90.17\scriptsize{$\pm$0.5} \\
  & 0.3 & 90.59\scriptsize{$\pm$0.4} & 90.08\scriptsize{$\pm$0.3}
        & 85.75\scriptsize{$\pm$1.8} & 85.61\scriptsize{$\pm$1.9} \\
  & 0.5 & 90.34\scriptsize{$\pm$0.3} & 89.12\scriptsize{$\pm$0.5}
        & 79.12\scriptsize{$\pm$4.0} & 77.65\scriptsize{$\pm$4.3} \\
  & 0.7 & 90.14\scriptsize{$\pm$0.6} & 87.64\scriptsize{$\pm$1.1}
        & 72.93\scriptsize{$\pm$5.9} & 68.31\scriptsize{$\pm$6.9} \\
\midrule
\multirow{4}{*}{Out-Scene}
  & 0.1 & 75.05\scriptsize{$\pm$1.7} & 75.06\scriptsize{$\pm$1.6}
        & 74.46\scriptsize{$\pm$1.6} & 74.47\scriptsize{$\pm$1.5} \\
  & 0.3 & 74.61\scriptsize{$\pm$1.5} & 74.69\scriptsize{$\pm$1.6}
        & 71.34\scriptsize{$\pm$1.7} & 71.49\scriptsize{$\pm$1.6} \\
  & 0.5 & 74.53\scriptsize{$\pm$1.6} & 74.18\scriptsize{$\pm$1.7}
        & 66.09\scriptsize{$\pm$1.9} & 65.86\scriptsize{$\pm$1.9} \\
  & 0.7 & 74.07\scriptsize{$\pm$1.8} & 73.60\scriptsize{$\pm$1.5}
        & 59.78\scriptsize{$\pm$2.1} & 57.05\scriptsize{$\pm$2.1} \\
\midrule
\multirow{4}{*}{Caltech}
  & 0.1 & 57.92\scriptsize{$\pm$1.7} & 57.88\scriptsize{$\pm$1.7}
        & 57.63\scriptsize{$\pm$1.8} & 57.62\scriptsize{$\pm$1.8} \\
  & 0.3 & 57.78\scriptsize{$\pm$1.8} & 57.77\scriptsize{$\pm$1.7}
        & 56.17\scriptsize{$\pm$2.1} & 56.23\scriptsize{$\pm$2.0} \\
  & 0.5 & 57.58\scriptsize{$\pm$1.7} & 57.63\scriptsize{$\pm$1.8}
        & 53.31\scriptsize{$\pm$2.9} & 53.59\scriptsize{$\pm$2.7} \\
  & 0.7 & 57.40\scriptsize{$\pm$1.7} & 57.35\scriptsize{$\pm$1.8}
        & 48.12\scriptsize{$\pm$4.5} & 47.74\scriptsize{$\pm$4.5} \\
\bottomrule
\end{tabular}
\end{table}

The same pattern holds on all three benchmarks: lenient
protocols (P1, P2) stay nearly flat as $r$ grows, while
stringent protocols (P3, P4) produce larger deployment-time
degradation as fewer views are retained. Caltech has larger
variation at high missing rates; the experiment does not isolate
its cause from view heterogeneity or the training recipe.

\subsection{YTF-31 ($V{=}5$, $K{=}31$)}
\label{app:results:ytf}

Table~\ref{tab:ytf_full} reports per-(protocol, $r$) accuracy on
YTF-31; the train-once advantage compresses a $16$-run DCP grid
($\sim$$16$ hours) into a single CRAFT run ($\sim$$3$ hours).

\begin{table}[h]
\centering
\caption{ACC (\%) on YTF-31 ($V{=}5$, $K{=}31$). Random baseline:
$3.2\%$. DCP trains $16$ times ($\sim$$16$ hours); CRAFT trains
once ($\sim$$3$ hours). Single seed (seed $42$) due to dataset
scale ($n \approx 100\text{k}$).}
\label{tab:ytf_full}
\small
\begin{tabular}{@{}ll rrrr@{}}
\toprule
Method & $r$ & Proto.~1 & Proto.~2 & Proto.~3 & Proto.~4 \\
\midrule
\multirow{4}{*}{DCP}
  & 0.1 & 32.97 & 33.29 & 27.95 & 33.85 \\
  & 0.3 & 29.49 & 30.86 & 16.00 & 12.73 \\
  & 0.5 & 28.05 & 33.04 & 10.80 & 11.17 \\
  & 0.7 & 27.57 & 30.66 & 5.79 & 6.08 \\
\midrule
\multirow{4}{*}{CRAFT}
  & 0.1 & 27.64 & 27.56 & 26.01 & 25.98 \\
  & 0.3 & 26.88 & 26.76 & 22.39 & 22.27 \\
  & 0.5 & 26.13 & 25.93 & 19.68 & 19.45 \\
  & 0.7 & 25.39 & 25.16 & 18.45 & 18.10 \\
\bottomrule
\end{tabular}
\end{table}

\subsection{Missing-Rate-Matched Training}
\label{app:results:mrmatched}

To compare training with all views available and rate-matched training on
HandWritten, we train a variant exposed to the
target missing rate during training, using hyperparameters
appropriate for this training regime
(\S\ref{sec:exp:setup}).

\begin{table}[h]
\centering
\caption{HandWritten ($V{=}6$): CRAFT trained with all views available vs.\
rate-matched CRAFT under Protocol~1 and Protocol~4. Both
variants disable masked fine-tuning; their training distributions
and hyperparameters differ. The canonical CRAFT$^{\star}$ row in
Table~\ref{tab:main} uses a separate recipe with MFT.
$\dagger$ = single-seed evaluation; all other entries are
5-seed mean.}
\label{tab:mr_matched}
\small
\setlength{\tabcolsep}{4pt}
\begin{tabular}{@{}l cccc cccc@{}}
\toprule
& \multicolumn{4}{c}{Protocol~1}
& \multicolumn{4}{c}{Protocol~4} \\
\cmidrule(lr){2-5} \cmidrule(lr){6-9}
$r$ & 0.1 & 0.3 & 0.5 & 0.7 & 0.1 & 0.3 & 0.5 & 0.7 \\
\midrule
All views available & 97.37 & 97.50 & 97.69 & 97.83
                 & 97.48 & 95.68 & 90.56 & 82.48 \\
Per-rate         & 94.75 & 94.59$^\dagger$ & 94.99 & 95.03
                 & 94.92 & 93.98 & 91.05 & 82.56 \\
\midrule
Gap              & $+2.62$ & $+2.91$ & $+2.70$ & $+2.80$
                 & $+2.56$ & $+1.70$ & $-0.49$ & $-0.08$ \\
\bottomrule
\end{tabular}
\end{table}

Across the 8 configurations of Table~\ref{tab:mr_matched}, the two variants
differ by at most $2.91$pp; the per-rate variant occasionally
outperforms the model trained with all views available at the most stringent configurations
(Protocol~4, $r \geq 0.5$). The small gaps describe these tested recipes; training
distribution and hyperparameters both vary, so this comparison
does not isolate architecture as the sole cause.

\subsection{Complete-MVC Comparison}
\label{app:results:complete}

\begin{table}[h]
\centering
\caption{Clustering performance with all views observed (ACC~\%).
$\dagger$ = numbers cited from original papers;
OOM = out of memory on a single RTX~3090 (24\,GB);
--- = not run.}
\label{tab:complete_full}
\small
\begin{tabular}{@{}l ccccccc@{}}
\toprule
Method & CUB & MF & UCI & OS & HW & Cal & YTF \\
\midrule
EPFMVC$^\dagger$ & --- & --- & --- & --- & 96.90 & --- & --- \\
\shortstack[l]{SparseMVC\\\citep{liu2026sparsemvc}}
  & 77.07{\scriptsize$\pm$1.1}
  & OOM
  & 90.27{\scriptsize$\pm$3.5}
  & 77.49
  & 90.26{\scriptsize$\pm$3.5}
  & 47.10{\scriptsize$\pm$1.1}
  & OOM \\
CRAFT
  & 83.80{\scriptsize$\pm$1.1}
  & 93.41{\scriptsize$\pm$0.3}
  & 91.46{\scriptsize$\pm$0.5}
  & 75.55{\scriptsize$\pm$0.8}
  & 97.49{\scriptsize$\pm$0.9}
  & 58.86{\scriptsize$\pm$1.2}
  & 28.10 \\
\bottomrule
\end{tabular}
\end{table}

CRAFT exceeds the listed complete-MVC baselines on CUB, UCI,
HandWritten, and Caltech, while SparseMVC is higher on Out-Scene
(Table~\ref{tab:complete_full}). The same CRAFT model is also
evaluated under missing views in the deployment grid.

\subsection{Training Time Details}
\label{app:results:time}

CRAFT starts from training data with all views available and reuses one final
checkpoint per dataset and seed for every (protocol, $r$) configuration,
while baselines retrain per configuration. On MultiFashion,
the recorded DVIMC grid cost is $16\times8=128$ minutes,
versus $14.5$ minutes for one CRAFT training run; their ratio is
$8.8\times$. A single DVIMC run is shorter than a single CRAFT
run. On YTF-31 the recorded DCP grid and CRAFT run take about
$16$ and $3$ hours, respectively. These are historical training-time
records; the retained timings do not itemize checkpoint selection,
hyperparameter search, or the separately described post-Stage~2 MFT.
They therefore do not establish a speedup for the complete development
pipeline or the new per-configuration training comparison.

\subsection{Extended Observation-Support Diagnostics}
\label{app:results:collapse}

\paragraph{Observation-support comparison configurations.}\label{app:results:isolation}%
Table~\ref{tab:isolation} reports three HandWritten configurations.
Setting~C uses Protocol~1 at $r=1$, so every sample retains five
views: $p_c=0$, but $q_2=1$. Settings~A and~C have different
effective missing rates ($11.7\%$ and $16.67\%$), so their
contrast does not isolate $p_c$ causally. DCP and DCG are trained
under the incomplete configurations; CRAFT and CRAFT-Core reuse
checkpoints whose training starts with all views available. The table is therefore
a support diagnostic with an explicit training/deployment distinction.

\begin{table}[h]
\centering
\caption{Observation-support comparison on HandWritten ($V{=}6$, $K{=}10$).
  ACC (\%); Setting~C DCP, DCG, and Energy-DIMC are 5-seed
  mean $\pm$ std. Setting~A/D DCP and DCG and all Energy-DIMC
  A/D entries are single-seed. CRAFT A/D reuse the five-seed
  deployment means; CRAFT/Core C have no established independent-run
  count, and Core D is a separate reduced-recipe result.
  Bold = best per setting;
  $^\ddagger \leq 40\%$ on $K{=}10$;
  $^\dagger$ on DCG = the A/D entries are single-seed;
  --- = not run.}
\label{tab:isolation}
\small
\setlength{\tabcolsep}{6pt}
\begin{tabular}{@{}l ccc@{}}
\toprule
& Setting~A & Setting~C & Setting~D \\
& Proto~1, $r{=}0.7$ & Proto~1, $r{=}1.0$ & Proto~4, $r{=}0.7$ \\
& $\hat{r}{=}11.7\%$ & $\hat{r}{\approx}16.67\%$ & $\hat{r}{=}70\%$ \\
& $p_c{=}30\%$ & $p_c{=}0$ & $p_c{\approx}0$ \\
\midrule
\multicolumn{4}{@{}l}{\textit{Class~1: $\mathcal{F}_{rec}$}} \\
DCP~\citep{lin2022dcp}
  & 71.55
  & 19.79{\scriptsize$\pm$2.42}$^\ddagger$
  & 16.25$^\ddagger$ \\
DCG$^\dagger$~\citep{zhang2025dcg}
  & 62.30
  & 12.29{\scriptsize$\pm$0.42}$^\ddagger$
  & 12.75$^\ddagger$ \\
\midrule
\multicolumn{4}{@{}l}{\textit{Class~2: cross-sample}} \\
Energy-DIMC~\citep{wang2025energydimc}
  & 96.37
  & 94.14{\scriptsize$\pm$1.07}
  & 83.26 \\
\midrule
\multicolumn{4}{@{}l}{\textit{Class~3: per-sample (C1+C2)}} \\
CRAFT-Core (Stage~1 only)
  & ---
  & 89.75
  & 54.83 \\
CRAFT (canonical)
  & \textbf{97.58}
  & \textbf{96.34}
  & \textbf{85.21} \\
\bottomrule
\end{tabular}
\end{table}

\paragraph{\texorpdfstring{$\mathcal{F}_{rec}$}{F\_rec} sub-families and the COMPLETER multi-view extension.}%
$\mathcal{F}_{rec}$ branches differ in whether reconstruction
requires fully-complete samples ($|O_i|=V$) or only pairwise
co-observation ($|O_i| \geq 2$). DCP and DCG belong to the
former; their reconstruction support in
Proposition~\ref{prop:trainability} is $p_c$, so this branch is
inactive on Setting~C. Their low ACC is a separate empirical
observation in Table~\ref{tab:isolation}. The multi-view extension of COMPLETER
used in this work instead adopts pairwise-complete training
(reconstruction and prediction losses are computed on any view
pair $(i,j)$ for which both views are observed in a given
sample). Under Min-1 (every sample retains $V{-}1$ views), each
sample contributes $\binom{V-1}{2}$ co-observed view pairs and
the eligible fraction is exactly $q_2(P)=1$;
pairwise reconstruction retains support, and COMPLETER (5-seed) retains
$68.76\pm 5.34\%$ on Setting~C. This result makes the distinction between complete-sample
and pairwise support explicit within $\mathcal{F}_{rec}$.

\paragraph{Per-baseline failure modes (CUB Protocol~4, $r{=}0.5$).}
Table~\ref{tab:collapse_cub} matches each baseline's observed
behavior at this configuration to the available diagnostic evidence.

\begin{table}[h]
\centering
\caption{CUB, Protocol~4, $r{=}0.5$ ($p_c{=}0$). Observed ACC and implementation diagnostics; these entries
do not identify a common failure cause. Bold = best ACC.}
\label{tab:collapse_cub}
\small
\begin{tabular}{@{}l r l@{}}
\toprule
Method & ACC (\%) & Diagnostic context \\
\midrule
DCG        & 15.61 & Recovery gradient vanishes (Prop.~\ref{prop:trainability}) \\
ProImp     & crash & Fixed-topology crash \\
COMPLETER  & 36.43 & Partial recovery dependence \\
DCP        & 36.83 & Partial recovery dependence \\
HSACC      & 38.63{\scriptsize$\pm$1.5} & Partial recovery dependence \\
Energy-DIMC & 42.38 & Distributional objective; low ACC \\
CRAFT      & \textbf{68.99}{\scriptsize$\pm$3.8} & Per-sample independence + variable-length fusion \\
\bottomrule
\end{tabular}
\end{table}

\paragraph{Energy-DIMC four-protocol comparison (HandWritten, $r{=}0.7$).}

\begin{table}[h]
\centering
\caption{HandWritten ($V{=}6$), $r{=}0.7$. Energy-DIMC retrains
per rate (single seed). CRAFT: one canonical attention\_mask
checkpoint with MFT, $5$-seed mean. Bold = best per protocol.}
\label{tab:energydimc_hw}
\small
\begin{tabular}{@{}l cccc@{}}
\toprule
Method & Proto.~1 & Proto.~2 & Proto.~3 & Proto.~4 \\
\midrule
Energy-DIMC & 96.37 & 95.65 & 83.98 & 83.26 \\
CRAFT
  & \textbf{97.58}\scriptsize{$\pm$0.9}
  & \textbf{97.49}\scriptsize{$\pm$0.9}
  & \textbf{86.03}\scriptsize{$\pm$2.1}
  & \textbf{85.21}\scriptsize{$\pm$2.4} \\
\bottomrule
\end{tabular}
\end{table}

CRAFT leads at every protocol on HandWritten
(Table~\ref{tab:energydimc_hw}), with protocol-dependent gaps: from $1.21$ points at Protocol~1 to $1.95$ points at Protocol~4.
The gap actually widens between these endpoints. The comparison
records performance under reduced observation; it does not show
that either method reaches an information-theoretic limit.

\subsection{Full Comparison with BURG}
\label{app:results:burg}

Under Protocol~2---BURG's inferred protocol
(\S\ref{app:details:burg})---CRAFT leads
BURG at every tested rate on CUB and stays within ${\sim}0.2\%$
of its lenient ceiling on HandWritten across $r$, while BURG
itself develops a ${\sim}15\%$ cliff between $r{=}0.5$ and
$r{=}0.7$ on HandWritten (full per-rate numbers in
Table~\ref{tab:cub_full} and Table~\ref{tab:hw6v_full}).

\subsection{Fusion Architecture Spectrum (Full Tables)}
\label{app:results:fusion_spectrum}

Table~\ref{tab:fusion_spectrum_full} reports the architectural
variant comparison summarized in \S\ref{sec:exp:ablation} across all
four protocols on HandWritten ($V{=}6$). The variants share
per-view encoders + Stage~1 + Stage~2 training without MFT for the within-table fusion comparison, differing only in the fusion
block: (i)~Transformer with attention masking (the CRAFT fusion block,
strict C2); (ii)~\emph{SetMLP}, a DeepSets-style
fusion~\citep{zaheer2017deepsets} consisting of a per-view shared
MLP, a masked mean pool over observed views, and a post-pool MLP
of hidden dimension $h'$ (strict C2), evaluated at two capacity
tiers \emph{SetMLP-Light} ($h'{=}192$) and
\emph{SetMLP-Match} ($h'{=}560$, fixed
across datasets); and (iii)~a zero-fill concat fusion baseline
(soft~C2). DVIMC uses a different training recipe and has recorded
initialization failures
(Appendix~\ref{app:details:dvimc}). Canonical CRAFT in
Table~\ref{tab:hw6v_full} uses a separate recipe with MFT. 5-seed mean
$\pm$ std.

\begin{table}[h]
\centering
\caption{Fusion architecture spectrum on HandWritten ($V{=}6$):
ACC (\%) under all four protocols. Bold = best per configuration.}
\label{tab:fusion_spectrum_full}
\small
\setlength{\tabcolsep}{4pt}
\begin{tabular}{@{}ll rrrr@{}}
\toprule
Architecture & $r$ & Proto.~1 & Proto.~2 & Proto.~3 & Proto.~4 \\
\midrule
\multirow{4}{*}{\shortstack[l]{Transformer\\+ attn mask\\(no MFT)}}
  & 0.1 & \textbf{97.28}\scriptsize{$\pm$1.0} & \textbf{97.28}\scriptsize{$\pm$1.0} & \textbf{97.11}\scriptsize{$\pm$1.0} & \textbf{97.14}\scriptsize{$\pm$1.1} \\
  & 0.3 & \textbf{97.25}\scriptsize{$\pm$1.1} & \textbf{97.26}\scriptsize{$\pm$1.1} & \textbf{95.31}\scriptsize{$\pm$1.6} & \textbf{95.59}\scriptsize{$\pm$1.5} \\
  & 0.5 & \textbf{97.24}\scriptsize{$\pm$1.1} & \textbf{97.13}\scriptsize{$\pm$1.1} & \textbf{90.02}\scriptsize{$\pm$3.0} & \textbf{90.79}\scriptsize{$\pm$3.0} \\
  & 0.7 & \textbf{97.24}\scriptsize{$\pm$1.0} & \textbf{97.09}\scriptsize{$\pm$1.1} & \textbf{80.29}\scriptsize{$\pm$5.1} & \textbf{79.06}\scriptsize{$\pm$5.7} \\
\midrule
\multirow{4}{*}{\shortstack[l]{SetMLP-Light\\($h'{=}192$)}}
  & 0.1 & 94.80\scriptsize{$\pm$1.3} & 94.78\scriptsize{$\pm$1.3} & 93.83\scriptsize{$\pm$1.2} & 94.01\scriptsize{$\pm$1.2} \\
  & 0.3 & 94.54\scriptsize{$\pm$1.2} & 94.54\scriptsize{$\pm$1.3} & 88.23\scriptsize{$\pm$1.3} & 88.82\scriptsize{$\pm$1.3} \\
  & 0.5 & 94.37\scriptsize{$\pm$1.2} & 93.94\scriptsize{$\pm$1.3} & 78.19\scriptsize{$\pm$1.6} & 78.97\scriptsize{$\pm$1.8} \\
  & 0.7 & 94.17\scriptsize{$\pm$1.2} & 93.56\scriptsize{$\pm$1.3} & 64.38\scriptsize{$\pm$2.0} & 62.09\scriptsize{$\pm$2.3} \\
\midrule
\multirow{4}{*}{\shortstack[l]{SetMLP-Match\\($h'{=}560$)}}
  & 0.1 & 94.56\scriptsize{$\pm$1.6} & 94.53\scriptsize{$\pm$1.6} & 92.84\scriptsize{$\pm$1.5} & 92.93\scriptsize{$\pm$1.7} \\
  & 0.3 & 93.86\scriptsize{$\pm$1.6} & 93.92\scriptsize{$\pm$1.6} & 85.30\scriptsize{$\pm$1.7} & 85.89\scriptsize{$\pm$1.7} \\
  & 0.5 & 93.45\scriptsize{$\pm$1.6} & 93.11\scriptsize{$\pm$1.6} & 73.23\scriptsize{$\pm$1.6} & 73.83\scriptsize{$\pm$1.7} \\
  & 0.7 & 92.74\scriptsize{$\pm$1.6} & 92.19\scriptsize{$\pm$1.7} & 57.78\scriptsize{$\pm$1.9} & 55.15\scriptsize{$\pm$1.9} \\
\midrule
\multirow{4}{*}{\shortstack[l]{Concat fusion\\(fixed input,\\soft C2)}}
  & 0.1 & 94.58\scriptsize{$\pm$0.3} & 94.53\scriptsize{$\pm$0.3} & 93.31\scriptsize{$\pm$0.3} & 93.38\scriptsize{$\pm$0.4} \\
  & 0.3 & 94.17\scriptsize{$\pm$0.2} & 94.27\scriptsize{$\pm$0.3} & 86.13\scriptsize{$\pm$0.4} & 86.89\scriptsize{$\pm$0.2} \\
  & 0.5 & 93.98\scriptsize{$\pm$0.3} & 93.60\scriptsize{$\pm$0.4} & 74.09\scriptsize{$\pm$0.4} & 75.04\scriptsize{$\pm$0.4} \\
  & 0.7 & 93.63\scriptsize{$\pm$0.3} & 92.93\scriptsize{$\pm$0.3} & 58.72\scriptsize{$\pm$0.6} & 56.55\scriptsize{$\pm$0.7} \\
\bottomrule
\end{tabular}
\end{table}

\textbf{Three findings.}
\emph{(1)~Architecture dominance.} The tested Transformer
dominates both SetMLP tiers and the fixed-input concat fusion
baseline across every
(protocol, rate) configuration, with the gap widening monotonically under
stringent regimes; at the hardest tested condition the three
strict-C2 architectures and the soft-C2 violator span
${\sim}24$ points.
\emph{(2)~A larger SetMLP does not close the gap.} Increasing SetMLP's hidden
size from $h'{=}192$ to $h'{=}560$ (parameters comparable to the
concat fusion baseline) does not narrow the gap to the
Transformer; the larger SetMLP-Match \emph{underperforms} the
smaller SetMLP-Light by $5{-}7\%$ under hard missingness. Thus increasing this hidden width alone does not improve
these results; other capacity and optimization effects remain
possible.
\emph{(3)~The concat-fusion gap grows with $r$.} In the
reported configurations,
SetMLP-Light's lead over the fixed-input concat fusion baseline
grows from a fraction of a point at $r{=}0.1$ to roughly
$5$--$6\%$ at $r{=}0.7$ under
Protocol~4. This comparison changes the fusion block and cannot assign
the full gap specifically to placeholder leakage.

The Transformer's query-dependent aggregation performs better
than the tested shared-MLP pooling recipes. This supports the
choice of fusion block empirically, without making it the unique
architecture defined by C1+C2.

\textbf{DVIMC implementation comparison.}
DVIMC's HandWritten grid includes initialization failures and
non-monotone accuracy across successful runs; details appear in
Appendix~\ref{app:details:dvimc}. CRAFT completes every condition
in its deployment grid. The implementations and training recipes
differ, so this comparison does not prove C2 necessity.

\section{Ablation and Sensitivity Details}
\label{app:ablation}

\subsection{Full Ablation Tables}
\label{app:ablation:full}

All ablation experiments use $3$~seeds and report Protocol~4 ACC
(mean$\pm$std) unless otherwise noted. Each column removes one
component from the within-table ``Full'' reference. All rows use
the v4 learnable-placeholder mask strategy to hold mask handling
fixed across ablations, so absolute values are not directly
comparable to the canonical CRAFT$^{\star}$ attention-masking results in
Table~\ref{tab:main}
(see Appendix~\ref{app:ablation:mask} for mask strategy).

\begin{table}[h]
\centering
\caption{Loss and training-stage ablation (Protocol~4 ACC~\%).
Bold = ablation column outperforming Full.}
\label{tab:ablation_loss_full}
\small
\begin{tabular}{@{}l r cccccc@{}}
\toprule
& $r$
& Full
& \shortstack{No\\Stage~2}
& \shortstack{No\\Entropy}
& \shortstack{Recon\\Only}
& \shortstack{Repr\\Only}
& \shortstack{No\\Pretrain} \\
\midrule
\multirow{4}{*}{HW}
  & 0.1 & 94.23\scriptsize{$\pm$1.1} & 94.20\scriptsize{$\pm$1.0}
        & 94.12\scriptsize{$\pm$1.0} & 93.66\scriptsize{$\pm$1.5}
        & 16.01\scriptsize{$\pm$2.4} & 62.45\scriptsize{$\pm$3.8} \\
  & 0.3 & 87.63\scriptsize{$\pm$0.9} & 87.09\scriptsize{$\pm$0.9}
        & 86.96\scriptsize{$\pm$0.9} & 86.90\scriptsize{$\pm$3.8}
        & 12.56\scriptsize{$\pm$0.4} & 53.56\scriptsize{$\pm$5.3} \\
  & 0.5 & 74.42\scriptsize{$\pm$1.3} & 73.62\scriptsize{$\pm$1.5}
        & 73.47\scriptsize{$\pm$1.6} & 74.22\scriptsize{$\pm$5.8}
        & 11.33\scriptsize{$\pm$0.1} & 43.15\scriptsize{$\pm$6.4} \\
  & 0.7 & 54.27\scriptsize{$\pm$1.9} & 53.67\scriptsize{$\pm$1.9}
        & 53.20\scriptsize{$\pm$1.7} & \textbf{54.83}\scriptsize{$\pm$6.8}
        & 10.46\scriptsize{$\pm$0.1} & 31.40\scriptsize{$\pm$4.4} \\
\midrule
\multirow{3}{*}{CUB}
  & 0.1 & 76.86\scriptsize{$\pm$3.7} & 73.84\scriptsize{$\pm$1.7}
        & 74.26\scriptsize{$\pm$1.8} & \textbf{79.99}\scriptsize{$\pm$4.8}
        & 50.74\scriptsize{$\pm$3.2} & 51.80\scriptsize{$\pm$1.6} \\
  & 0.3 & 68.62\scriptsize{$\pm$4.4} & 65.63\scriptsize{$\pm$1.9}
        & 66.03\scriptsize{$\pm$2.1} & \textbf{73.93}\scriptsize{$\pm$4.8}
        & 31.60\scriptsize{$\pm$1.8} & 49.34\scriptsize{$\pm$1.3} \\
  & 0.5 & 61.54\scriptsize{$\pm$5.7} & 57.99\scriptsize{$\pm$2.5}
        & 58.42\scriptsize{$\pm$2.7} & \textbf{69.28}\scriptsize{$\pm$5.5}
        & 15.34\scriptsize{$\pm$1.1} & 47.49\scriptsize{$\pm$3.3} \\
\bottomrule
\end{tabular}
\end{table}

\textbf{Stage~1 pre-training is essential.}\quad
Skipping Stage~1 entirely (\textsc{No Pretrain}) drops accuracy by
$23{-}34\%$ on HW and roughly $14\%$ on CUB at the hardest tested
rates (Table~\ref{tab:ablation_loss_full}). Stage~1's role is to
learn representations that pre-align with the Stage~2 cluster
head; without it, Stage~2 must drive both representation learning
and clustering simultaneously, and the evaluated clustering-only recipe
performs worse under missing-view evaluation.

\textbf{Reconstruction is the load-bearing Stage~1 signal.}\quad
Replacing Stage~1 with consistency-loss-only training
(\textsc{Repr Only}, $\lambda_{\mathrm{recon}}{=}0$) collapses
performance to near-random on HW (at most $16.01\%$ on $K{=}10$, where
random is $10\%$) and to $15{-}51\%$ on CUB. The reconstruction
loss provides the bulk of the cluster-relevant gradient; the
consistency loss alone, despite being per-sample and C1-compliant,
cannot anchor the representation in the absence of a per-view
reconstruction target in this evaluated recipe.

\textbf{Consistency loss has dataset-dependent effects.}\quad
Removing the consistency loss (\textsc{Recon Only},
$\lambda_{\mathrm{repr}}{=}0$) leaves accuracy approximately
unchanged on HW ($\leq 1\%$ gap to Full at every $r$) but
\emph{improves} accuracy on CUB by $3{-}8\%$
(Table~\ref{tab:ablation_loss_full}). View count and dataset identity vary together here, so the
asymmetry cannot be attributed to view count alone. Consistency is therefore a
dataset-dependent regularizer, not a load-bearing component.

\textbf{Stage~2 components contribute marginally.}\quad
Removing the Stage~2 fine-tuning (\textsc{No Stage~2}) or its
entropy regularizer (\textsc{No Entropy}) reduces accuracy by at
most ${\sim}3.5\%$ at every condition tested
(Table~\ref{tab:ablation_loss_full}). Both contribute to the tested recipe, with larger effects on CUB
than on HandWritten.

\textbf{Additional MFT has dataset-dependent effects.}\quad
Masked fine-tuning (\S\ref{sec:method:mft}) yields modest
single-digit improvements on HW and MultiFashion at high missing
rates and a small degradation on CUB. The asymmetry follows the
view-subset ensemble interpretation of MFT
(Appendix~\ref{app:proofs:mft}): at $V \geq 3$, MFT mixes many
distinct view subsets and acts as a regularizer; at $V=2$, the stated sampling rule retains all views.
The observed changes therefore include further fine-tuning and
any feature augmentation, rather than a view-deletion effect.

\subsection{Reduced Core (CRAFT-Core)}
\label{app:ablation:core}

CRAFT-Core retains attention masking and a Stage~1
reconstruction loss, removing SimSiam consistency, entropy and
KL regularizers, MFT, and the Stage~2 cluster-head fine-tuning.
A single Core checkpoint per dataset and seed is evaluated under all four
protocols at four rates with attention masking applied at
inference; Table~\ref{tab:ablation_core} reports the result.
Core retains useful accuracy on both datasets and exceeds the
listed reconstruction baselines on CUB. On HandWritten it does
not win every configuration. At Protocol~4, $r=0.7$, Core reaches
$51.94\%$ versus canonical CRAFT's $85.21\%$. This comparison
removes several components jointly, so their individual effects
cannot be separated from the gap. The $54.83\%$ reduced-recipe
entry in the support diagnostic belongs to the separate original
ablation recipe and is not this Core checkpoint.

\begin{table}[h]
\centering
\caption{CRAFT-Core ACC (\%) on CUB and HandWritten under all
four protocols. 3-seed mean $\pm$ std.}
\label{tab:ablation_core}
\small
\begin{tabular}{@{}ll rrrr@{}}
\toprule
Dataset & $r$ & Proto.~1 & Proto.~2 & Proto.~3 & Proto.~4 \\
\midrule
\multirow{4}{*}{CUB}
  & 0.1 & 79.96\scriptsize{$\pm$3.5} & 79.79\scriptsize{$\pm$3.4} & 78.36\scriptsize{$\pm$2.9} & 79.01\scriptsize{$\pm$3.6} \\
  & 0.3 & 76.93\scriptsize{$\pm$3.0} & 77.40\scriptsize{$\pm$3.6} & 73.58\scriptsize{$\pm$2.6} & 72.34\scriptsize{$\pm$3.0} \\
  & 0.5 & 74.30\scriptsize{$\pm$2.9} & 74.33\scriptsize{$\pm$3.0} & 69.98\scriptsize{$\pm$2.6} & 66.93\scriptsize{$\pm$3.2} \\
  & 0.7 & 71.47\scriptsize{$\pm$2.4} & 70.91\scriptsize{$\pm$2.2} & 67.19\scriptsize{$\pm$2.1} & 66.93\scriptsize{$\pm$3.2} \\
\midrule
\multirow{4}{*}{HW}
  & 0.1 & 93.63\scriptsize{$\pm$0.9} & 93.60\scriptsize{$\pm$1.0} & 92.39\scriptsize{$\pm$1.6} & 92.29\scriptsize{$\pm$1.7} \\
  & 0.3 & 93.35\scriptsize{$\pm$1.1} & 93.27\scriptsize{$\pm$1.2} & 85.08\scriptsize{$\pm$2.8} & 85.46\scriptsize{$\pm$3.4} \\
  & 0.5 & 92.92\scriptsize{$\pm$1.5} & 92.68\scriptsize{$\pm$1.6} & 71.90\scriptsize{$\pm$3.8} & 72.34\scriptsize{$\pm$3.8} \\
  & 0.7 & 92.81\scriptsize{$\pm$1.8} & 92.03\scriptsize{$\pm$1.7} & 54.50\scriptsize{$\pm$3.3} & 51.94\scriptsize{$\pm$3.2} \\
\bottomrule
\end{tabular}
\end{table}

\subsection{Sensitivity Figures}
\label{app:ablation:sensitivity}

\begin{figure}[h]
\centering
\includegraphics[width=\textwidth]{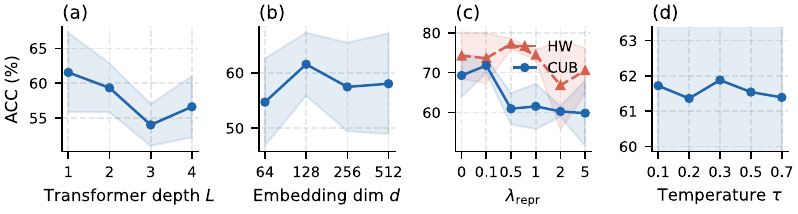}
\caption{Sensitivity to four key hyperparameters under Protocol~4
$r{=}0.5$. \textbf{(a)}~Depth: $L{=}1$ optimal; deeper fusion
amplifies overfitting. \textbf{(b)}~Embedding dim: $d{=}128$
optimal for CUB ($n{=}600$); larger $d$ overfits.
\textbf{(c)}~$\lambda_{\mathrm{repr}}$: dataset-dependent (HW
optimum at $0.5$, CUB at $0.1$, consistent with the asymmetry in
Table~\ref{tab:ablation_loss_full}). \textbf{(d)}~Temperature: insensitive
($<\!0.6$ percentage-point spread over $\tau\in[0.1,0.7]$). Shaded bands:
$\pm 1$ std over 3 seeds.}
\label{fig:sensitivity}
\end{figure}

Figure~\ref{fig:sensitivity} reports the four sweeps at the
canonical evaluation point (Protocol~4, $r{=}0.5$). We re-ran
each sweep at
$r \in \{0.1, 0.3, 0.7\}$ and confirmed that the optimal settings
remain stable: $L{=}1$ is optimal at every tested missing rate;
$d{=}128$ is optimal for CUB at every rate (larger $d$ overfits
on $n{=}600$); the $\lambda_{\mathrm{repr}}$ dataset-dependence
(HW optimum near $0.5$, CUB optimum near $0.1$) is preserved
across rates; the temperature sweep over $\tau\in[0.1,0.7]$ has an ACC
spread below $0.6$ percentage points. The relative ordering of configurations
in each panel does not change with $r$, so the canonical
configuration in Appendix~\ref{app:details:hyperparam} is
appropriate at all four nominal rates.

\subsection{Missing-View Strategy Comparison}
\label{app:ablation:mask}

The complete comparison is reported in
Table~\ref{tab:ablation_mask} in \S\ref{sec:exp:ablation}.

In Table~\ref{tab:ablation_mask}, the strict-C2 strategy
(attention masking) consistently dominates soft-C2 alternatives
(learnable, zero, mean), with the gap widening at higher missing
rates. This is consistent with
Proposition~\ref{prop:attn_bound}: under attention masking,
absent views contribute zero attention weight exactly; under
soft-C2, the bound allows nonzero attention to missing positions. The
ACC differences in this table remain empirical observations.

\subsection{Aggregation Strategy}
\label{app:ablation:agg}

\begin{table}[h]
\centering
\caption{Aggregation strategy comparison (Protocol~4 ACC~\%).
Bold = best per row.}
\label{tab:ablation_agg}
\small
\begin{tabular}{@{}l r ccc@{}}
\toprule
Dataset & $r$ & CLS & Mean & Max \\
\midrule
\multirow{4}{*}{HW}
  & 0.1 & \textbf{94.23}\scriptsize{$\pm$1.1} & 93.85\scriptsize{$\pm$0.1}
        & 93.45\scriptsize{$\pm$0.6} \\
  & 0.3 & 87.63\scriptsize{$\pm$0.9} & \textbf{88.72}\scriptsize{$\pm$0.5}
        & 84.85\scriptsize{$\pm$2.1} \\
  & 0.5 & 74.42\scriptsize{$\pm$1.3} & \textbf{75.61}\scriptsize{$\pm$1.5}
        & 71.33\scriptsize{$\pm$3.7} \\
  & 0.7 & \textbf{54.27}\scriptsize{$\pm$1.9} & 53.62\scriptsize{$\pm$2.7}
        & 51.69\scriptsize{$\pm$4.0} \\
\midrule
\multirow{3}{*}{CUB}
  & 0.1 & \textbf{76.86}\scriptsize{$\pm$3.7} & 75.50\scriptsize{$\pm$4.0}
        & 76.79\scriptsize{$\pm$3.0} \\
  & 0.3 & 68.62\scriptsize{$\pm$4.4} & 68.10\scriptsize{$\pm$3.7}
        & \textbf{68.98}\scriptsize{$\pm$3.0} \\
  & 0.5 & 61.54\scriptsize{$\pm$5.7} & \textbf{61.97}\scriptsize{$\pm$4.2}
        & 61.92\scriptsize{$\pm$2.8} \\
\bottomrule
\end{tabular}
\end{table}

In Table~\ref{tab:ablation_agg}, the three aggregation strategies
(CLS / Mean / Max) differ by at most $4.28$ percentage points,
indicating that the C1+C2 properties are carried by the attention
mechanism itself, not by the choice of pooling head.

\subsection{Attention Analysis}
\label{app:ablation:attention}

On CUB ($V{=}2$), with one view missing the learnable placeholder
receives approximately $38\%$ of total CLS attention ($76\%$ of
the uniform share), higher than the $57\%$ ratio observed on
HandWritten ($V{=}6$) because the shorter sequence leaves fewer
real tokens to redistribute weight to. Under canonical attention
masking, the missing position receives exactly zero weight
regardless of the number of views
(Proposition~\ref{prop:attn_bound}, $\varepsilon = -\infty$);
Figure~\ref{fig:attn_reallocation} visualizes the placeholder
attention reallocation pattern on HandWritten as a function of
the number of missing views.

\begin{figure}[h]
  \centering
  \includegraphics[width=0.7\textwidth]{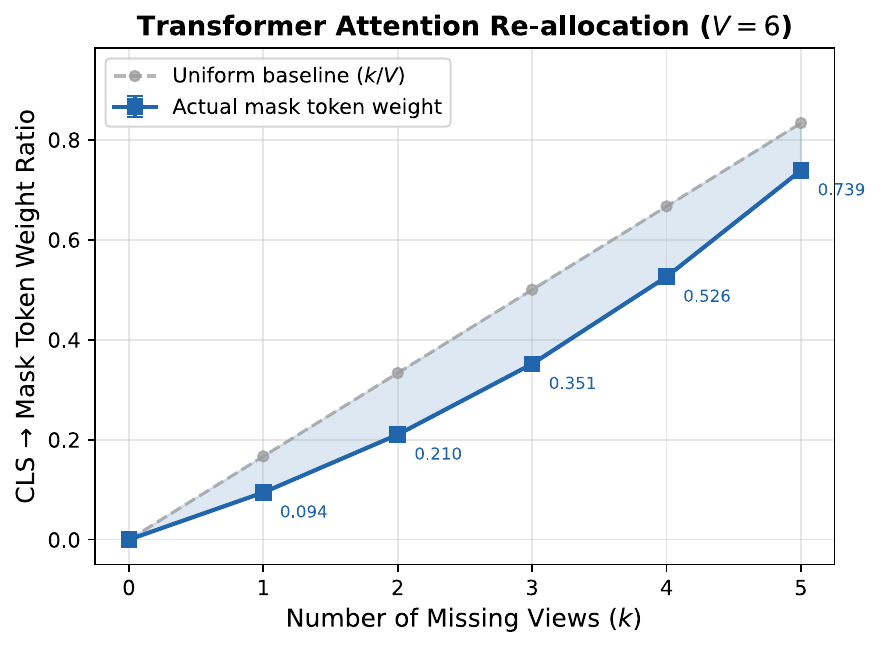}
  \caption{CLS$\to$placeholder attention weight on HandWritten
    ($V{=}6$, learnable-placeholder strategy) vs.\ the number of
    missing views $k$. Trained model (solid) vs.\ uniform
    baseline $k/V$ (dashed).}
  \label{fig:attn_reallocation}
\end{figure}


\section{Extended Discussion}
\label{app:discussion}

\subsection{Implications and Practical Recommendations}
\label{app:discussion:implications}

Cross-view reconstruction provides useful supervision when its
required views are observed together. Our analysis bounds the
gradient contribution of this branch under explicit support and
normalization conditions. Three practical recommendations follow:
report $\hat r$ and $p_c$ alongside nominal rate $r$, and the
relevant pairwise support where needed; check which loss branches
remain active on the training masks; and assess incomplete training
separately from reusing a model whose training starts with all views available. C1+C2 guide
forward computation across observed subsets, while successful
learning depends on the objective and the available information.

\subsection{Reproducibility Statement}
\label{app:discussion:repro}

Code, configurations, and a four-protocol evaluation toolkit are
released with the paper; CRAFT inference under arbitrary
missing-view patterns uses only PyTorch's standard
\texttt{MultiheadAttention} key-padding-mask interface, with no
custom module required.

\end{document}

%% file: authors.tex
\author{%
  Haolu Liu$^{1}$ \quad Xiyue Wang$^{1}$ \quad Xuanting Xie$^{1}$ \\
  \textbf{Liangjian Wen$^{2}$ \quad Zhao Kang$^{1}$\thanks{Corresponding author: Zhao Kang.}} \\
  $^{1}$School of Computer Science and Engineering, \\
  University of Electronic Science and Technology of China \\
  $^{2}$School of Computing and Artificial Intelligence, \\
  Southwestern University of Finance and Economics \\
  \texttt{\{2023080136006,2024080902018\}@std.uestc.edu.cn} \\
  \texttt{x624361380@outlook.com}\quad
    \texttt{wenlj@swufe.edu.cn}\quad\texttt{zkang@uestc.edu.cn}
}

%% file: section_abstract.tex
\begin{abstract}
Incomplete multi-view clustering (IMVC) is typically evaluated by retraining separate models under different missing-view configurations. Evaluations indexed only by nominal missing rate can overlook differences in observation structure across missing-view protocols. We show that missing-data protocols with identical nominal missing rates can induce substantially different learning regimes, differing by approximately $50\times$ in the proportion of fully observed samples. We formalize this phenomenon as \textbf{protocol divergence}, which quantifies structural disparities among missing-view protocols beyond marginal missing rates. Furthermore, we analyze support-gated reconstruction mechanisms and show that their optimization contribution is inherently limited by the frequency of eligible observations under explicit normalization and optimization conditions. Based on these observations, we propose \textbf{CRAFT} (Co-occurrence-free Robust Attention-masked Fusion Transformer), a train-once framework that combines representation learning with an architecture designed to process missing-view inputs. CRAFT combines (i) per-sample forward computation using each sample's observed views and shared parameters, and (ii) mask-aware fusion over nonempty observed-view subsets. The deployment evaluation starts from training data with all views available and reuses one final checkpoint per dataset and seed across missing-view protocols without retraining. Experiments on CUB and MultiFashion show that CRAFT achieves the strongest performance in 12 out of 13 information-matched settings. Additional deployment experiments across seven benchmarks and sixteen missing configurations demonstrate substantial computational savings through checkpoint reuse while maintaining competitive clustering performance.
Code and evaluation tools: \href{https://github.com/dk23lhl/CRAFT}{github.com/dk23lhl/CRAFT} and \href{https://github.com/dk23lhl/imvc-audit}{github.com/dk23lhl/imvc-audit}.
\end{abstract}

%% file: section_1_intro.tex
\section{Introduction}
\label{sec:intro}

Multi-view representations have become increasingly important for learning from heterogeneous data collected by multiple sensors, modalities, or feature sources. By exploiting complementary information across views, multi-view clustering (MVC) has demonstrated strong potential for unsupervised representation learning and data discovery. However, the complete-view assumption underlying most MVC methods rarely holds in real-world applications. Sensor failures, privacy constraints, and transmission errors frequently result in missing views for a subset of samples, motivating the study of \emph{Incomplete Multi-View Clustering (IMVC)}~\citep{wen2022survey}.

Existing IMVC methods primarily address missing views through cross-view reconstruction, where observed views are leveraged to recover unavailable ones and facilitate subsequent clustering. Representative approaches introduce reconstruction networks that exploit correlations among views, often using complete samples as the primary source of cross-view supervision~\citep{lin2021completer,zhang2025dcg,hsacc2025}. Comparisons indexed by nominal missing rate alone can obscure differences in the observation structure that supplies this supervision.

In this work, we show that missing rates alone are insufficient to characterize the difficulty of IMVC. Specifically, we reveal that two missing-view protocols with identical nominal missing rates may induce fundamentally different learning regimes because they contain different proportions of fully observed samples and different observation structures. We refer to this phenomenon as \textbf{protocol divergence}. Comparisons across protocols therefore need to account for realized observation structure, even when each method is retrained for each configuration.

\textbf{Per-Configuration Training and Train-Once Deployment.}
The dominant evaluation protocol in IMVC trains and tests a separate model for each predefined combination of missing pattern and missing rate~\citep{lin2022dcp,zhang2025dcg,hsacc2025}. Such evaluation measures whether a model can adapt to a known missing distribution after retraining. However, practical deployment often follows a different scenario: a model is trained once, while the missing-view pattern may change dynamically after deployment due to varying sensor availability or environmental conditions. These two settings address fundamentally different questions. The former evaluates adaptation under retraining, whereas the latter evaluates robustness to changing missing-view configurations without additional optimization.

\textbf{The Complete-Sample Supervision Bottleneck.}
Beyond evaluation protocols, we identify a structural limitation in reconstruction-based IMVC methods. Reconstruction objectives requiring all views rely on support-gated supervision, where a reconstruction branch is activated only when the required views are simultaneously available. Let $p_c$ denote the proportion of fully observed samples. Although two protocols may share the same nominal missing rate, their values of $p_c$ can differ substantially; for example, under six views and a nominal missing rate of $0.5$, the complete-sample proportion may vary by approximately $50\times$ across protocols. Consequently, reconstruction-based optimization receives different amounts of effective supervision under apparently identical missing rates. We theoretically characterize this effect by deriving bounds on the gradient contribution and finite-budget optimization displacement of support-gated reconstruction objectives under explicit normalization and optimization conditions.

These observations motivate separating the supervision used to learn representations from the architecture used to process missing-view inputs at deployment. Ideally, a model should satisfy two properties: (i) for fixed shared parameters, each sample representation should depend only on its own observed views; and (ii) the fusion mechanism should explicitly account for any nonempty observed-view subset.

To achieve this goal, we propose \textbf{CRAFT} (Co-occurrence-free Robust Attention-masked Fusion Transformer), a train-once framework for incomplete multi-view clustering. CRAFT combines two architectural principles for processing missing-view inputs. First, \textbf{per-sample independence} ensures that, for fixed shared parameters, each sample representation depends on its own observed views rather than other samples' observation patterns. Second, \textbf{mask-aware variable-length fusion} aggregates only observed views for any nonempty view subset. These properties allow a shared model to process changing observed-view subsets. In the default deployment setting, CRAFT implements the ``Train-Once, Deploy-Many'' paradigm (Figure~\ref{fig:paradigm}) by starting from training data with all views available and reusing one final checkpoint per dataset and seed across missing-view protocols without retraining.

Our main contributions are summarized as follows:

\textbf{(1) Revealing and formalizing protocol divergence.}
We demonstrate that nominal missing rates fail to capture the structural difficulty of incomplete multi-view learning. We introduce effective missing rate $\hat{r}$ and complete-sample proportion $p_c$ to quantify disparities among missing-view protocols beyond marginal missing rates.

\textbf{(2) Theoretical analysis of reconstruction-based optimization.}
We analyze support-gated reconstruction mechanisms and establish bounds on their gradient contribution and finite-budget SGD displacement, revealing how incomplete observation structures constrain optimization.

\textbf{(3) A train-once architecture for robust IMVC.}
We propose CRAFT, a Transformer-based framework with per-sample independence and mask-aware variable-length fusion, supporting reuse of a fixed checkpoint across changing missing-view configurations.

\textbf{(4) Comprehensive empirical validation.}
Experiments on CUB and MultiFashion demonstrate that CRAFT achieves the strongest performance in 12 out of 13 information-matched comparisons. Further deployment experiments across seven benchmarks and sixteen missing-view configurations verify the effectiveness of checkpoint reuse and demonstrate substantial reductions in repeated-training cost.

%% file: section_2_background.tex
\section{Related Work}\label{sec:bg:imvc}

Incomplete multi-view clustering has a long pre-deep
history---early methods addressed the problem through low-rank
matrix factorization~\citep{li2014pvc,hu2018daimc,shao2015mic},
incomplete multi-kernel learning~\citep{liu2017mkkmik},
spectral and graph-based clustering~\citep{wang2019pic,zhao2016img},
and subspace alignment with missing-view
inferring~\citep{wen2019ueaf}, each compensating for missing
entries through a structural prior on the recovered
representation. Recent tensor-completion methods additionally model
cross-order neighborhoods and cross-view
diversity~\citep{liu2026incomplete}.

Beyond missing-view recovery, related clustering work studies
scalable single- and multi-view subspace
graphs~\citep{kang2022structured,kang2020large},
multi-view attributed graphs~\citep{lin2023multiview},
high-order relations in generic multi-view data~\citep{pan2023highorder},
and graph filtering for real-world and complex
data~\citep{xie2026provable,kang2025cdc}.

Deep IMVC methods inherit the
cross-view supervision idea but train the representation
end-to-end. We organize the discussion around three mechanisms
concerning supervision and forward computation (Class~1, Class~2,
Class~3 below). These groups are not mutually exclusive: a per-sample
architecture can also be trained with support-gated reconstruction losses.

\textbf{Class~1: cross-view reconstruction ($\mathcal{F}_{rec}$).} The dominant paradigm trains explicit cross-view recovery modules using co-observed views as supervision. COMPLETER~\citep{lin2021completer}
and its TPAMI extension DCP~\citep{lin2022dcp} establish the
template; subsequent methods refine the recovery mechanism through
cognitive deep networks~\citep{wen2020cdimcnet}, generative
recovery via GANs and diffusion~\citep{wang2021gpmvc,zhang2025dcg},
prototype-based imputation~\citep{li2023proimp}, hierarchical
semantic alignment~\citep{hsacc2025}, attention-based partial deep
clustering~\citep{xu2023apadc}, and recent Bayesian
extensions~\citep{burg2025}. The reconstruction loss is computed over pairs of views
co-observed for the same sample, so this branch receives direct supervision
only from samples retaining the required view pair. We denote these
support-gated reconstruction losses by $\mathcal{F}_{rec}$;
additional losses in the same method may use other samples.

\textbf{Class~2: cross-sample and distributional methods.}
A second class extracts supervisory signal that is not strictly
per-sample: DSIMVC~\citep{tang2022dsimvc} imputes from latent-space
nearest neighbors of \emph{other} samples,
DSMVC~\citep{tang2022dsmvc} reweights views to mitigate
performance degradation when adding views,
SURE~\citep{yang2023sure} combines cross-sample contrastive
alignment with noise-robust losses for partially view-aligned and
sample-missing settings, and
Energy-DIMC~\citep{wang2025energydimc} aligns view distributions
through energy-based models. All four compute their training signal
through cross-sample or cross-view distributional relationships
rather than within-sample co-observation, placing these objectives outside the
$\mathcal{F}_{rec}$ gradient bound. Their performance depends
on optimization and representation quality as well as observed
information.

\textbf{Class~3: per-sample methods.} A third class computes each sample's fused representation from its own observed views
and shared parameters, with no dependence on other samples'
observation patterns. DVIMC~\citep{chen2025dvimc} uses a
variational autoencoder with a Wasserstein-barycenter mixture for
per-sample latent fusion;
MvCLN~\citep{yang2021mvcln},
FreeCSL~\citep{freecsl2025},
and I2MVC~\citep{i2mvc2025}
similarly fuse only observed views without cross-view recovery.
CRAFT instantiates this per-sample forward design through
attention-masked variable-length fusion, while using reconstruction
and consistency objectives during representation learning.
Outside clustering, attention-based fusion of observed-only
modalities is established in supervised multimodal
learning~\citep{ma2022multimodal,he2022mae}, where labels provide supervision that unsupervised IMVC must
instead derive from the available data.

%% file: section_3_problem.tex
\section{Motivation and Insights}
\label{sec:protocol}

\paragraph{Notation.} We consider $n$ samples with up to $V$ views and $K$ latent clusters $Y \in \{1, \ldots, K\}$. View-$v$ feature $x_v^{(i)} \in \mathbb{R}^{d_v}$ is observed iff $M_v^{(i)}=1$, where $M_v^{(i)} \in \{0,1\}$. The observed-view set is $O_i := \{v : M_v^{(i)} = 1\}$; sample $i$ is \emph{complete} iff $|O_i| = V$. The protocol uses nominal missing rate $r \in [0, 1]$; \S\ref{sec:protocol:vars} gives the realized statistics on $M$.

\subsection{Generalized Evaluations}
\label{sec:protocol:vars}

The two protocol-induced statistics introduced in
\S\ref{sec:intro}---the effective missing rate $\hat{r}$ and
the complete-sample proportion $p_c$---admit closed forms over the
mask matrix $M$:
\begin{equation}
\hat{r} \;=\; \frac{1}{nV}\sum_{i=1}^{n}\sum_{v=1}^{V}\bigl(1 - M_v^{(i)}\bigr),
\qquad
p_c \;=\; \frac{1}{n}\sum_{i=1}^{n}\prod_{v=1}^{V} M_v^{(i)}.
\label{eq:metrics}
\end{equation}
Both summarize the protocol's effect on the data but describe
different aspects of the observation structure.

\textbf{$p_c$ describes complete-sample support.}\quad
Reconstruction branches requiring all views receive supervision
only from complete samples. Pairwise branches instead use samples
with $|O_i|\geq2$, whose proportion can remain positive when
$p_c=0$. The required observation support therefore depends on the loss.

\textbf{$\hat{r}$ describes view coverage.}\quad
Even a perfectly trained method is limited by the information
available in observed views. The effective missing rate summarizes
coverage but does not determine which information is retained:
two protocols can share $\hat{r}$ yet differ in both $p_c$ and
the identity of the observed views. We refer to the joint $(\hat{r}, p_c)$
characterization as \emph{protocol divergence}, since two
protocols realizing the same nominal $r$ may diverge sharply
along either axis.

\subsection{Four Evaluation Protocols and \texorpdfstring{$50\times$}{50x} Divergence in $p_c$}
\label{sec:protocol:four}

Protocols~1 and~2 have effective missing rates
$\hat{r}\leq1/V$. Protocol~1 can still reach $p_c=0$ at
$r=1$, when each sample retains $V-1$ views. We introduce two
complementary protocols to cover a wider range of observation
structures: Protocol~3 uses per-view independent masking followed
by one-view restoration for empty samples; Protocol~4 protects
one view per sample and deletes a fixed total number of the
remaining entries uniformly without replacement.

Three orthogonal design choices distinguish the protocols used in
the IMVC literature~\citep{lin2021completer,lin2022dcp,hsacc2025}:
\begin{itemize}
\itemsep -1pt
\item \emph{Granularity}: whether masking decisions are made per
sample or over entries of the mask matrix;
\item \emph{Protected view}: whether one view is selected before deletion
and kept throughout masking;
\item \emph{One-hot guarantee}: whether an empty sample has one view
restored after masking. All four protocols here retain at least one view.
\end{itemize}

These choices yield four canonical protocols
(Table~\ref{tab:protocols}), summarized by their
limiting $\hat{r}$ and $p_c$ as $n \to \infty$ (derivations in
Appendix~\ref{app:protocol:formulas}, Monte Carlo verification in
Appendix~\ref{app:protocol:mc}):

\begin{table}[!htb]
\centering
\caption{Four evaluation protocols and their limiting hidden
  variables ($\hat{r}$, $p_c$) as $n \to \infty$.
  Over the reported grid, Protocols~3--4 reach lower $p_c$
  and higher effective missing rates than Protocols~1--2.
  $V$ is the total view count and $r$ the nominal missing rate.}
\label{tab:protocols}
\setlength{\tabcolsep}{4pt}
\resizebox{0.85\linewidth}{!}{%
\begin{tabular}{@{}c l c c c@{}}
\toprule
Protocol & Mechanism & Granularity & $\hat{r}$ & $p_c$ \\
\midrule
1 & Sample-level deletion (protected) & sample
  & $r/V$ & $1 - r$ \\
2 & Entry-wise Bernoulli (protected) & entry
  & $r/V$ & $\bigl(1 - r/(V{-}1)\bigr)^{V-1}$ \\
3 & Per-view independent (one-hot) & entry
  & $r - r^V/V$ & $(1 - r)^V$ \\
4 & Fixed-budget deletion (protected) & entry
  & $\min\bigl(r,\, (V{-}1)/V\bigr)$
  & $\bigl[\max(0,\,1 - Vr/(V{-}1))\bigr]^{V-1}$ \\
\bottomrule
\end{tabular}%
}
\end{table}

\textbf{What drives the divergence.}\quad
The protocols combine different deletion budgets and allocation
rules. Protocols~1--2 remove $r/V$ of all entries on average,
whereas Protocols~3--4 permit larger effective missing rates;
entry-wise allocation further changes the chance of retaining
all views. The result, visible in
Figure~\ref{fig:protocol_pc}, partitions the family into a
\emph{lenient} regime (Protocols~1, 2; $p_c \geq 30\%$ throughout)
and a \emph{stringent} regime (Protocols~3, 4). For example,
Protocol~1 and Protocol~4 differ by about $50\times$ in $p_c$
at $V{=}6$, $r{=}0.5$; larger ratios occur elsewhere in the grid.

\begin{figure}[!htb]
\centering
\includegraphics[width=\textwidth]{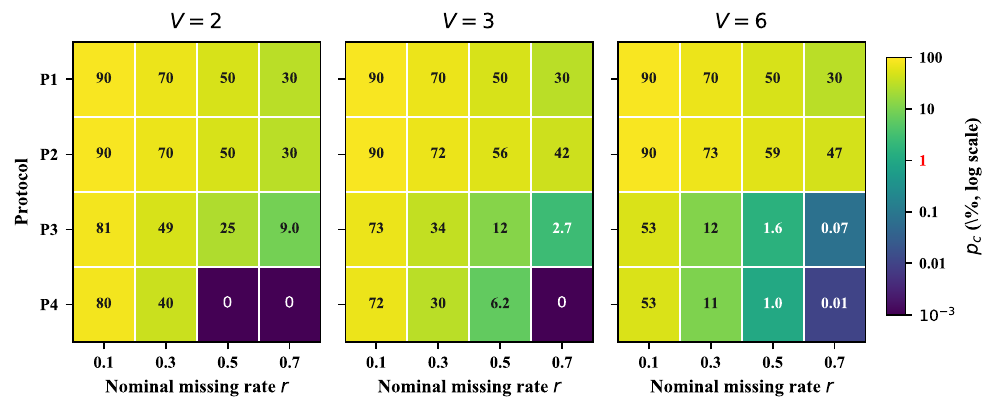}
\caption{Complete-sample proportion $p_c$ (\%, log color scale)
under the four protocols, across $V \in \{2, 3, 6\}$ and
$r \in \{0.1, 0.3, 0.5, 0.7\}$. The $1\%$ tick (red) on the
colorbar provides a reference level for low complete-sample support.}
\label{fig:protocol_pc}
\end{figure}

\textbf{Coverage gap.}\label{sec:protocol:gap}
All nine audited IMVC implementations
(Appendix~\ref{app:protocol:audit}) concentrate on Protocols~1--2 or their variants. At
$r{=}0.5$, their limiting $p_c$ is at least $50\%$ for
$V{=}6$. This coverage gap motivates testing lower-support
regimes. CRAFT is
excluded from this baseline protocol audit.

%% file: section_4_theory.tex
\section{Theoretical Analysis}
\label{sec:theory}

We formalize two architectural conditions on the encoder
$f_\theta$ that produces fused representations
$h^{(i)} = f_\theta(\{x_v^{(i)}\}_{v \in O_i}) \in \mathbb{R}^d$
from observed views, and separately analyze support-gated reconstruction
objectives under explicit normalization and optimization assumptions;
the empirical loss is
$\mathcal{L}(\theta) = \tfrac{1}{n}\sum_i \ell(h^{(i)}, x^{(i)})$.
Full SGD-trajectory notation and theorem proofs appear in
Appendix~\ref{app:proofs}.

\paragraph{Condition C1 (Per-sample independence).}
$h^{(i)}$ depends only on sample~$i$'s observed views and shared
parameters, with no dependence on the observation patterns of
other samples in its forward computation. Batch-level training
regularizers are considered separately.

\paragraph{Condition C2 (Mask-aware variable-length fusion).}
$f_\theta$ is defined for any subset
$O_i \subseteq \{1,\ldots,V\}$ with $|O_i| \geq 1$, producing
$h^{(i)}$ via a mechanism that aggregates only observed views,
including implementations using fixed slots with explicit masking.

\begin{assumption}[Conditional independence given $Y$]
\label{ass:cond_ind}
With $X_v$ denoting the random variable underlying $x_v^{(i)}$,
$P(X_1, \ldots, X_V \mid Y) = \prod_{v=1}^V P(X_v \mid Y).$
\end{assumption}

Throughout, $\mathcal{F}_{rec}$ denotes the Class~1 reconstruction
loss family of \S\ref{sec:bg:imvc}. We first state a universal
capability bound that applies to every method regardless of
architecture; the trainability results that follow are
independent of this conditional-independence assumption, with
Assumption~\ref{ass:cond_ind} required only by the second inequality of Theorem~\ref{thm:capability}.

\begin{theorem}[Capability bound]
\label{thm:capability}
For any clustering function taking a nonempty set of observed views
$S \subseteq \{1,\ldots,V\}$ as input, the bound
\[
  \max_{v \in S} I(X_v; Y)
   \;\leq\; I\bigl(\{X_v\}_{v \in S};\, Y\bigr)
\]
holds unconditionally. Under Assumption~\ref{ass:cond_ind},
the additional bound
$I\bigl(\{X_v\}_{v \in S};\, Y\bigr)
   \leq \sum_{v \in S} I(X_v;\, Y)$
also applies (proof in Appendix~\ref{app:proofs:thm1b}).
\end{theorem}

For a learned representation, data processing gives
$I(h;Y)\leq I(X_S;Y)$. The retained information depends on
representation learning and optimization; this capability ceiling
does not identify a method-specific failure cause.

\begin{proposition}[Trainability bound for $\mathcal{F}_{rec}$]
\label{prop:trainability}
For a reconstruction branch normalized over all sampled examples,
write its per-sample loss as
$\ell_i=\mathbf{1}\{|O_i|\geq2\}\tilde\ell_i$.
If $\|\nabla\tilde\ell_i\|\leq C_g$ along the trajectory and
examples are sampled uniformly, then
$\mathbb{E}\|\nabla\ell_i\|\leq C_g q_2(P)$, where
$q_2(P)=\mathbb{P}_{M\sim P}[|O|\geq2]$.
The same upper bound holds for the mean gradient of a uniformly
sampled mini-batch. Full statement and normalization details are in
Appendix~\ref{app:proofs:thm1a}.
\end{proposition}

For a strict-complete branch, replace the indicator by
$\mathbf{1}\{|O_i|=V\}$ and the bound becomes $C_g p_c$.
In general, $p_c\leq q_2(P)$, with equality for $V=2$;
at $V>2$, $p_c=0$ can coexist with $q_2(P)=1$.
At zero eligible support this branch has zero gradient, while
other loss terms may still update the model.

\begin{theorem}[Finite-budget displacement within $\mathcal{F}_{rec}$]
\label{thm:fr_necessity}
Under Proposition~\ref{prop:trainability}'s conditions, suppose
constant-step SGD updates only this loss with step size $\eta$,
using fresh iid samples for $T$ steps. Then
$\mathbb{E}\|\theta_T-\theta_0\|\leq\eta T C_g q_2(P)$.
For strict-complete losses, replace $q_2(P)$ by $p_c$.
The proof is in Appendix~\ref{app:proofs:thm2}.
\end{theorem}

\paragraph{Role of C1+C2.}\label{cor:c1c2_sufficient}
C1 and C2 define a broad family of forward architectures that
accept nonempty observed subsets without complete-sample input
dependencies. A training objective that remains active on such
subsets avoids the corresponding reconstruction support gate.
These conditions alone do not guarantee successful learning.

%% file: section_5_method.tex
\section{CRAFT}
\label{sec:method}

CRAFT instantiates the two conditions of \S\ref{sec:theory}: 
C1 via per-sample forward representations, and C2 via a Transformer with attention 
masking~\citep{vaswani2017transformer,ma2022multimodal} 
that admits variable-length input in the spirit of Set 
Transformer~\citep{lee2019settransformer} and 
Perceiver~\citep{jaegle2021perceiver}, specialized to 
unsupervised IMVC. These conditions define an architectural
family; CRAFT is one implementation. Attention provides
query-dependent weighting of observed views, while a shared
cluster head applies the same clustering rule to every subset.
The fusion comparison in Appendix~\ref{app:results:fusion_spectrum}
provides empirical context for this choice.

We distinguish two tiers. \textbf{CRAFT-Core} is the 
reduced configuration---Stage~1 reconstruction, cosine
cluster head, and attention masking. \textbf{CRAFT} (canonical) 
adds consistency loss, masked fine-tuning, entropy and KL 
regularizers, and Stage~2 fine-tuning to improve clustering,
particularly at high $V$ with few observed views
(Figure~\ref{fig:arch}; ablation in
Appendix~\ref{app:ablation:core}).

\begin{figure}[!htb]
\centering
\includegraphics[width=\textwidth]{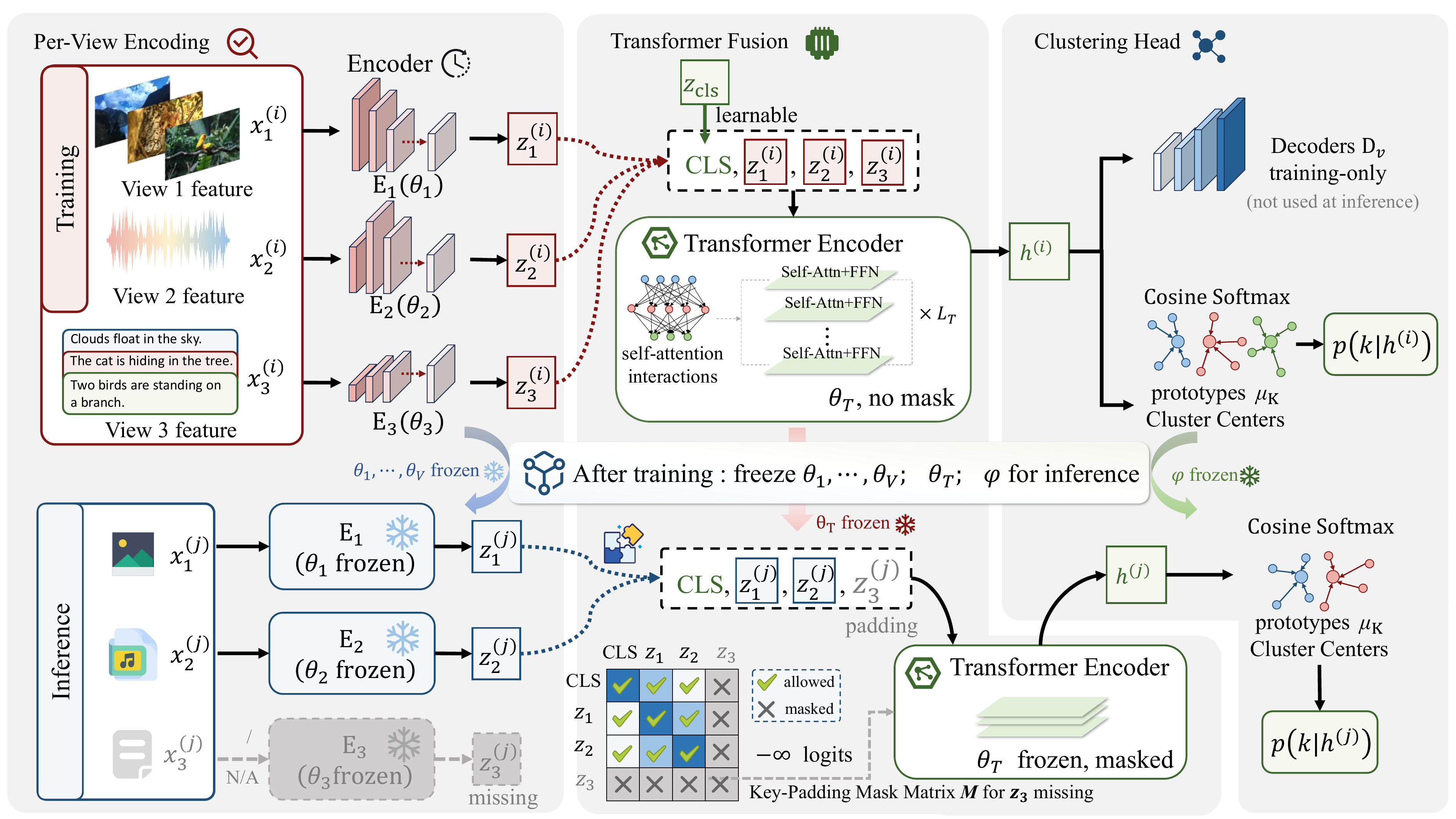}
\caption{CRAFT architecture: per-view encoders, Transformer
fusion with a learnable $\mathrm{[CLS]}$ token, and a cosine
softmax cluster head. \textbf{Top:}~default deployment training (Stages~1 and~2) with all views available; optional masked fine-tuning follows
(\S\ref{sec:method:mft}).
\textbf{Bottom:}~the trained model at inference with missing views
($z_3$ shown missing), excluded by a key-padding mask. The depicted modalities ($V{=}3$,
image/audio/text) and the masked-out view are illustrative
only; CRAFT supports arbitrary view types and nonempty observed-view
subsets at inference. Hyperparameters in
Appendix~\ref{app:details:hyperparam}.}
\label{fig:arch}
\end{figure}

\subsection{Transformer Architecture}
\label{sec:method:arch}

\textbf{Per-view encoding.}\quad
View-specific MLP encoders $E_v$ map each view to a shared
$d$-dimensional embedding
$z_v^{(i)} = E_v(x_v^{(i)}) \in \mathbb{R}^d$.
A learnable $\mathrm{[CLS]}$ token is prepended to the sequence of
view embeddings, yielding $[\mathrm{[CLS]}, z_1^{(i)}, \ldots,
z_V^{(i)}]$ for each sample.

\textbf{Masked self-attention.}\quad
A Transformer block with multi-head self-attention is applied to the
token sequence. For sample $i$ at evaluation time with observed-view
set $O_i \subseteq \{1, \ldots, V\}$, the key-padding mask is set so
that positions corresponding to $v \notin O_i$ receive attention
logit $-\infty$. The softmax denominator excludes these positions,
ensuring missing views contribute exactly zero to the attention
weights and to the resulting fused representation
(Appendix~\ref{app:proofs:exclusion}). The $\mathrm{[CLS]}$ token's
output $h^{(i)}$ serves as the fused sample representation.

\textbf{Cluster head.}\quad
A cosine softmax classifier maps the fused representation
$h^{(i)}$ to a soft cluster assignment over $K$
clusters. The classifier weights are $L_2$-normalized prototypes;
classification logits are scaled cosine similarities.

\subsection{Two-Stage Training for Train-Once Deployment}
\label{sec:method:training}

The default deployment model starts from training data in which
every sample has all $V$ views available.
Stages~1 and~2 use all views, followed by optional additional
masked fine-tuning. The final model is then frozen for evaluation
under all missing-data configurations. The separate information-matched
training conditions are specified in \S\ref{sec:exp:setup}.

\textbf{Stage 1: representation learning.}\quad
The Stage 1 loss combines a per-view reconstruction term and a
cross-view consistency term:
\begin{equation}
\mathcal{L}_{\mathrm{stage1}}(\theta)
\;=\; \underbrace{\frac{1}{n}\sum_{i=1}^{n}\sum_{v=1}^{V}
        \bigl\|D_v(h^{(i)}) - x_v^{(i)}\bigr\|^2}_{
        \mathcal{L}_{\mathrm{recon}}}
   + \lambda_{\mathrm{repr}}\,
     \mathcal{L}_{\mathrm{repr}}(\{z_v^{(i)}\}_{v}),
\label{eq:stage1}
\end{equation}
where $D_v$ is a view-specific decoder MLP and
$\mathcal{L}_{\mathrm{repr}}$ is a SimSiam-style consistency loss
on the per-view embeddings of the same
sample~\citep{chen2021simsiam}
(formal definitions:
Appendix~\ref{app:details:losses}). Both terms use views of the same sample,
while the forward representation satisfies C1. Under this default
training recipe, reconstruction and consistency provide
representation-learning supervision before deployment. The frozen model
subsequently processes observed views through the masked fusion described above.

\textbf{Stage 2: cluster fine-tuning.}\quad
The Stage 1 representation is not cluster-aligned by default: $h$
captures multi-view structure, but its geometry need not place
samples of the same cluster close to one another. Stage 2 adds a
cluster head $g_\phi$ and a clustering objective:
\begin{equation}
\mathcal{L}_{\mathrm{stage2}}(\theta, \phi)
\;=\; \mathcal{L}_{\mathrm{cluster}}(g_\phi(h^{(i)}))
   + \beta\, \mathcal{L}_{\mathrm{ent}}
   + \gamma\, \mathcal{L}_{\mathrm{KL}},
\label{eq:stage2}
\end{equation}
Here $\mathcal{L}_{\mathrm{cluster}}$ is a self-labeled
cross-entropy on cosine softmax outputs with epoch-refreshed
pseudo-labels, $\mathcal{L}_{\mathrm{ent}}$ is an anti-collapse
entropy regularizer on batch-averaged predictions, and
$\mathcal{L}_{\mathrm{KL}}$ enforces consistency between
predictions from MFT-sampled view-subsets
(\S\ref{sec:method:mft}); $\gamma{=}0$ during ordinary Stage~2
(formal definitions:
Appendix~\ref{app:details:losses};
full procedure: Algorithm~\ref{alg:craft},
Appendix~\ref{app:details:algorithm}). The entropy term aggregates
predictions across a batch; the forward representation remains
per-sample as specified by C1.

\subsection{Masked Fine-Tuning}
\label{sec:method:mft}

After Stage~2, optional \emph{masked fine-tuning} (MFT) continues
from its checkpoint for additional epochs using Eq.~\ref{eq:stage2}.
For each sample in each batch, draw
$k \sim \mathrm{Uniform}(\{0, \ldots, V-2\})$ and randomly exclude
$k$ views via the same key-padding mask used at inference. This
exposes the model to observed-view subsets through the same masking
mechanism used at inference and acts as a view-subset ensemble
(Appendix~\ref{app:proofs:mft}). For $V \geq 3$, the rule can sample
multiple subsets retaining at least two views; for $V = 2$, only the
full-view subset contributes, so this sampling rule performs no view deletion. Any feature noise is a separate augmentation.

%% file: section_6_1_setup.tex
\section{Experiments}
\label{sec:experiments}

\input{tables/matched_cub_mf}

Our experiments answer four questions in turn:
(i) how does CRAFT compare under matched incomplete training?
(ii) how do observation support and the choice of recovery loss
relate to baseline behavior?
(iii) how well does a single final CRAFT checkpoint cover the
deployment grid, and at what cost?
(iv) how do individual CRAFT components affect clustering performance?
Cross-protocol comparison (\S\ref{sec:exp:main}) addresses
(i)--(iii); the component ablation
(\S\ref{sec:exp:ablation}) addresses (iv).

\subsection{Setup}
\label{sec:exp:setup}

\textbf{Baselines and protocol grid.}\quad
We evaluate CRAFT against the three classes of
\S\ref{sec:bg:imvc}: Class~1
(COMPLETER~\citep{lin2021completer}, DCP~\citep{lin2022dcp},
DCG~\citep{zhang2025dcg}, HSACC~\citep{hsacc2025}), Class~2
(Energy-DIMC~\citep{wang2025energydimc}), and Class~3
(DVIMC~\citep{chen2025dvimc}), across the protocol-rate grid of
\S\ref{sec:protocol:four}. Three datasets are reported in
the main text; four more, along with hyperparameters, hardware,
seed protocol and metric details, appear in
Appendix~\ref{app:details:datasets}. The deployment tables use the canonical configuration of
\S\ref{sec:method}.

\textbf{Training and evaluation conditions.}\quad
The information-matched CUB and MultiFashion comparison trains
each method separately for each listed $(P,r)$ condition, using
matched observed training information, preprocessing, model-selection
information, and evaluation masks. CRAFT uses observed inputs and
targets, with consistency over eligible pairs
(Appendix~\ref{app:details:hyperparam}).
Separately, the train-once deployment grid starts CRAFT from training
data with all views available and reuses one final checkpoint per dataset
and seed across all sixteen configurations; baselines retrain per condition.
The HandWritten rate-matched results are included in
Table~\ref{tab:main}, with a comparison disabling MFT in both variants in
Appendix~\ref{app:results:mrmatched}. Independent-run counts and
implementation exceptions are given in Appendix~\ref{app:details:datasets}.

%% file: tables/matched_cub_mf.tex
\begin{table}[!tb]
\centering
\caption{Cross-protocol comparison. ACC (\%) on three datasets.
Baselines and CRAFT (per-condition) train separately per $(P,r)$
configuration; CRAFT$^{\star}$ (train-once) starts from training data
with all views available and reuses one final checkpoint per dataset
and seed across the sixteen-condition grid (Protocols~1/4 shown).
HandWritten uses MFT only in the train-once row.
Bold/underline = best/second per column within each dataset.
$\dagger$ = one effective seed for the marked method or cell;
unmarked CRAFT and HSACC entries are five-seed means. Other run counts
follow Appendix~\ref{app:details:datasets}. N/A = not reported;
--- = runtime failure; $\ddagger$ marks ACC $\leq40\%$.}
\label{tab:main}
\label{tab:matched_cub_mf}
\footnotesize
\setlength{\tabcolsep}{4pt}
\renewcommand{\arraystretch}{0.95}
\resizebox{\linewidth}{!}{%
\begin{tabular}{@{}l rrrr rrrr@{}}
\toprule
& \multicolumn{4}{c}{Protocol~1}
& \multicolumn{4}{c}{Protocol~4} \\
\cmidrule(lr){2-5} \cmidrule(lr){6-9}
Method & 0.1 & 0.3 & 0.5 & 0.7 & 0.1 & 0.3 & 0.5 & 0.7 \\
\midrule
\multicolumn{9}{@{}l}{\textit{CUB ($V{=}2$, $K{=}10$)}} \\[1pt]
DCG$^\dagger$~\citep{zhang2025dcg}
  & 71.50 & 72.56 & 73.33 & 63.50
  & 70.28 & 65.39 & 15.61$^\ddagger$ & 17.72$^\ddagger$ \\
COMPLETER~\citep{lin2021completer}
  & 51.33 & 61.60 & 49.83 & 18.90$^\ddagger$
  & 61.23 & 18.27$^\ddagger$ & 36.43$^\ddagger$ & 36.43$^\ddagger$ \\
DCP~\citep{lin2022dcp}
  & 65.00 & 57.50 & 51.33 & 26.00$^\ddagger$
  & 64.17 & 29.67$^\ddagger$ & 36.83$^\ddagger$ & 36.83$^\ddagger$ \\
HSACC~\citep{hsacc2025}
  & 63.13 & 54.00 & 43.40 & 29.67$^\ddagger$
  & 59.50 & 33.13$^\ddagger$ & 38.63$^\ddagger$ & 38.63$^\ddagger$ \\
Energy-DIMC~\citep{wang2025energydimc}
  & 75.20 & 67.77 & 58.20 & 51.56
  & 71.29 & 53.52 & \underline{42.38} & \underline{42.38} \\
DVIMC$^\dagger$~\citep{chen2025dvimc}
  & 22.67$^\ddagger$ & 27.67$^\ddagger$ & 22.67$^\ddagger$ & 22.67$^\ddagger$
  & 19.83$^\ddagger$ & 26.17$^\ddagger$ & --- & --- \\
CRAFT (per-condition)
  & \textbf{87.17} & \underline{77.50} & \underline{75.50} & \underline{69.00}
  & \underline{79.83} & \underline{72.00} & \textit{N/A} & \textit{N/A} \\
\cmidrule(lr){1-9}
CRAFT$^{\star}$ (train-once)
  & \underline{82.09} & \textbf{79.19} & \textbf{76.04} & \textbf{73.61}
  & \textbf{80.89} & \textbf{74.15} & \textbf{68.99} & \textbf{68.99} \\
\midrule
\multicolumn{9}{@{}l}{\textit{HandWritten ($V{=}6$, $K{=}10$)}} \\[1pt]
DCG$^\dagger$
  & 68.50 & 79.25 & 62.25 & 62.30
  & 65.25 & 32.60$^\ddagger$ & 14.45$^\ddagger$ & 12.75$^\ddagger$ \\
COMPLETER
  & 92.15 & 67.61 & 89.51 & 75.25
  & 85.06 & 89.31 & 58.34 & 12.65$^\ddagger$ \\
DCP
  & 82.35 & 83.90 & 80.00 & 71.55$^\dagger$
  & 82.10 & 24.60$^\ddagger$ & 18.15$^\ddagger$ & 16.25$^{\dagger\ddagger}$ \\
HSACC
  & 76.53 & 87.91 & 87.10 & 80.56
  & 89.39 & 15.05$^\ddagger$ & 14.24$^\ddagger$ & 13.52$^\ddagger$ \\
Energy-DIMC
  & \underline{96.65} & \underline{96.37} & \underline{96.09} & \underline{96.37}$^\dagger$
  & \underline{96.09} & \underline{94.87} & \underline{92.75} & \underline{83.26}$^\dagger$ \\
DVIMC$^\dagger$
  & 37.30$^\ddagger$ & 84.05 & 79.55 & 24.30$^\ddagger$
  & 71.40 & 20.00$^\ddagger$ & 77.05 & --- \\
CRAFT (per-condition)
  & 94.75 & 94.59$^\dagger$ & 94.99 & 95.03
  & 94.92 & 93.98 & 91.05 & 82.56 \\
\cmidrule(lr){1-9}
CRAFT$^{\star}$ (train-once)
  & \textbf{97.60} & \textbf{97.57} & \textbf{97.57} & \textbf{97.58}
  & \textbf{97.53} & \textbf{96.82} & \textbf{93.69} & \textbf{85.21} \\
\midrule
\multicolumn{9}{@{}l}{\textit{MultiFashion ($V{=}3$, $K{=}10$)}} \\[1pt]
DCG$^\dagger$
  & 96.29 & 91.99 & 85.65 & 76.25
  & 91.70 & 71.78 & 48.22 & 11.68$^\ddagger$ \\
COMPLETER
  & 79.92 & 87.23 & 90.66 & \underline{90.09}
  & 83.08 & 87.87 & 46.37 & --- \\
DCP
  & 86.04 & 72.75 & 72.26 & 68.95
  & 75.24 & 67.21 & 69.36 & 25.11$^\ddagger$ \\
HSACC
  & \underline{97.46} & \underline{96.40} & 91.67 & 83.36
  & \underline{94.34} & 82.50 & 49.80 & 21.61$^\ddagger$ \\
Energy-DIMC
  & 93.30 & 91.64 & 90.19 & 89.90
  & 91.97 & 88.53 & 79.77 & \underline{29.14}$^\ddagger$ \\
DVIMC
  & 83.86 & 86.62 & 88.82 & 87.21
  & 86.65 & 84.09 & 80.43 & --- \\
CRAFT (per-condition)
  & \textbf{98.07} & \textbf{97.32} & \textbf{96.69} & 89.61
  & \textbf{97.35} & \textbf{94.21} & \underline{86.59} & \textit{N/A} \\
\cmidrule(lr){1-9}
CRAFT$^{\star}$ (train-once)
  & 93.21 & 92.82 & \underline{92.44} & \textbf{91.97}
  & 92.73 & \underline{90.83} & \textbf{88.12} & \textbf{85.41} \\
\bottomrule
\end{tabular}%
}
\end{table}

%% file: section_6_2_main.tex
\subsection{Cross-Protocol Comparison}
\label{sec:exp:main}


\textbf{CRAFT leads most information-matched conditions.}\quad
Table~\ref{tab:matched_cub_mf} places CRAFT (per-condition) first in all six
CUB conditions and six of seven MultiFashion conditions. The
exception is MultiFashion under Protocol~1 at $r=0.7$, where
COMPLETER reaches $90.09\%$, followed by Energy-DIMC at
$89.90\%$ and CRAFT at $89.61\%$.

\textbf{Observation support depends on the reconstruction loss.}\quad
The HandWritten configuration using Protocol~1 at $r=1$ has
$p_c=0$ while every sample retains five views. Complete-sample
reconstruction is inactive, whereas pairwise reconstruction
retains support (Appendix~\ref{app:results:isolation}).

\textbf{A single final CRAFT checkpoint covers all sixteen deployment configurations.}\quad
Table~\ref{tab:main} shows the Protocol~1 and Protocol~4 portions
of this grid; the full results are in Appendix~\ref{app:results}.
CRAFT$^{\star}$ exceeds every baseline in the displayed CUB and HandWritten conditions,
while some MultiFashion conditions favor per-configuration
baselines. The HandWritten rate-matched control differs from its
no-MFT reference trained with all views available by at most $2.91$ points over the
reported conditions (Appendix~\ref{app:results:mrmatched}).

\textbf{Training cost and scale.}\quad
The recorded MultiFashion costs are $128$ minutes for a sixteen-run
DVIMC grid and $14.5$ minutes for one CRAFT run, a ratio of
$8.8\times$; these records do not itemize model selection or the
additional MFT stage (Appendix~\ref{app:results:time}).
For embedding dimension $d$ and $L$ Transformer layers, fixed-slot
fusion takes $O(L(V+1)^2d+L(V+1)d^2)$ time per sample, in addition
to the view encoders/decoders; attention storage is $O(L(V+1)^2)$
for a fixed head count. Training is linear in the number of samples
per epoch for fixed widths and view count. YTF-31, with $101{,}499$
samples, supplies the largest tested scale, using a single seed.

%% file: section_6_3_ablation.tex
\subsection{Component Ablation}
\label{sec:exp:ablation}

\textbf{The reconstruction loss is central to Stage~1; Stage~2 effects vary by dataset.}\quad
A subtractive ablation under Protocol~4 yields three qualitative
findings; per-rate numbers appear in
Appendix~\ref{app:ablation:full}. These three-seed ablations use
the original learnable-placeholder recipe and their within-table
Full reference. Consistency-only training performs poorly on
HandWritten, while removing reconstruction also lowers CUB accuracy.
Removing Stage~2 or entropy changes HandWritten accuracy by about
one point, with larger changes on CUB. Consistency and masked
fine-tuning have dataset-dependent effects; the CUB reconstruction-only
variant improves on its Full reference.

\textbf{CRAFT-Core preserves useful accuracy with a reduced training recipe.}\quad
CRAFT-Core exceeds the listed reconstruction baselines on CUB,
but on HandWritten it does not win every configuration. The gap
to canonical CRAFT grows at high missing rates
(Appendix~\ref{app:ablation:core}); this comparison removes several
components together and cannot identify MFT's isolated effect.
A four-hyperparameter sensitivity sweep
(Appendix~\ref{app:ablation:sensitivity}) supplies the tested
parameter ranges.

\textbf{Attention fusion outperforms the tested pooling and concat blocks.}\quad
Table~\ref{tab:fusion_main} compares fusion architectures under
Protocol~4; Appendix~\ref{app:results:fusion_spectrum} reports all four protocols.
The variants share the encoders and two-stage training without MFT,
and differ in the fusion block. The Transformer leads both SetMLP
widths and concat fusion at every reported rate. Increasing the
SetMLP width does not close this gap in the tested configurations.
This within-table comparison supports the fusion choice; its no-MFT
recipe is separate from canonical CRAFT in Table~\ref{tab:main}.

\input{tables/fusion_main}

\textbf{Excluding missing tokens improves the fixed-checkpoint comparison.}\quad
Table~\ref{tab:ablation_mask} compares four inference-time treatments
of missing views using the same checkpoint within each run, trained with
all views available. Attention masking has the highest mean ACC at all three
reported rates. The alternative strategies retain missing positions
in attention, whereas the mask excludes their contribution exactly
(Appendix~\ref{app:proofs:exclusion}). These results isolate inference-time
handling within this comparison, rather than a change in training.

\input{tables/missing_view_strategy}

%% file: tables/fusion_main.tex
\begin{table}[!htb]
\centering
\caption{Fusion comparison on HandWritten under Protocol~4.
ACC (\%), five-seed mean $\pm$ std. All variants share the encoders
and two-stage training without MFT, differing in their fusion block.
Bold = best per column. The full four-protocol grid is in
Appendix~\ref{app:results:fusion_spectrum}.}
\label{tab:fusion_main}
\small
\setlength{\tabcolsep}{5pt}
\begin{tabular}{@{}l rrrr@{}}
\toprule
Fusion & $r=0.1$ & $r=0.3$ & $r=0.5$ & $r=0.7$ \\
\midrule
Transformer + attn mask
 & \textbf{97.14}\scriptsize{$\pm$1.1}
 & \textbf{95.59}\scriptsize{$\pm$1.5}
 & \textbf{90.79}\scriptsize{$\pm$3.0}
 & \textbf{79.06}\scriptsize{$\pm$5.7} \\
SetMLP-Light ($h'=192$)
 & 94.01\scriptsize{$\pm$1.2}
 & 88.82\scriptsize{$\pm$1.3}
 & 78.97\scriptsize{$\pm$1.8}
 & 62.09\scriptsize{$\pm$2.3} \\
SetMLP-Match ($h'=560$)
 & 92.93\scriptsize{$\pm$1.7}
 & 85.89\scriptsize{$\pm$1.7}
 & 73.83\scriptsize{$\pm$1.7}
 & 55.15\scriptsize{$\pm$1.9} \\
Concat fusion
 & 93.38\scriptsize{$\pm$0.4}
 & 86.89\scriptsize{$\pm$0.2}
 & 75.04\scriptsize{$\pm$0.4}
 & 56.55\scriptsize{$\pm$0.7} \\
\bottomrule
\end{tabular}
\end{table}

%% file: tables/missing_view_strategy.tex
\begin{table}[!htb]
\centering
\caption{Missing-view strategy comparison on CUB under Protocol~4.
ACC (\%), three-seed mean $\pm$ std. The strategies use the same
checkpoint within each run, trained with all views available, differing only in
inference-time handling of missing-view tokens. Bold = best per row.}
\label{tab:ablation_mask}
\small
\setlength{\tabcolsep}{6pt}
\begin{tabular}{@{}r cccc@{}}
\toprule
$r$ & Attn Mask & Learnable & Zero & Mean \\
\midrule
0.1 & \textbf{77.29}\scriptsize{$\pm$4.1}
    & 76.86\scriptsize{$\pm$3.7}
    & 76.36\scriptsize{$\pm$3.4}
    & 75.77\scriptsize{$\pm$2.6} \\
0.3 & \textbf{69.78}\scriptsize{$\pm$5.4}
    & 68.62\scriptsize{$\pm$4.4}
    & 66.40\scriptsize{$\pm$3.3}
    & 65.03\scriptsize{$\pm$0.2} \\
0.5 & \textbf{63.84}\scriptsize{$\pm$7.9}
    & 61.54\scriptsize{$\pm$5.7}
    & 58.13\scriptsize{$\pm$4.4}
    & 55.89\scriptsize{$\pm$0.9} \\
\bottomrule
\end{tabular}
\end{table}

%% file: section_7_discussion.tex
\section{Conclusion}
\label{sec:discussion}

We introduced CRAFT, which combines representation learning with
per-sample, mask-aware fusion for deployment under changing observed-view
subsets. Our analysis bounds
the contribution of support-gated reconstruction branches,
distinguishing complete-sample from pairwise observation support.
By combining per-sample representations and mask-aware fusion,
CRAFT handles changing observed subsets with a shared clustering rule.
The information-matched CUB and MultiFashion comparison places
CRAFT first in twelve of thirteen conditions. Separately, deployment training
starts from data with all views available and yields one final checkpoint
per dataset and seed for the deployment grid, with lower
recorded cumulative training cost. Our work suggests that IMVC
research should report observation structure alongside nominal
missing rate and evaluate both incomplete training and model reuse.
The four controlled protocols expose structural differences in
synthetic missingness; naturally occurring, quality-dependent and
correlated sensor failures remain untested. Training when an entire
modality is never observed, and strictly zero-complete-sample training
beyond the reported conditions, require further study.

%% file: theorem2_proof.tex
%
%
%
%

\subsection{Proof of Theorem~\ref{thm:fr_necessity}
  (Finite-Budget Displacement within \texorpdfstring{$\mathcal{F}_{rec}$}{F_rec})}
\label{app:proofs:thm2}

This appendix proves Theorem~\ref{thm:fr_necessity} of the main text
using the gradient-counting and parameter-displacement steps below.
The result concerns a support-gated reconstruction loss under SGD.

\subsubsection{Notation and Setting}
\label{app:proofs:thm2:setup}

Let $\Theta \subseteq \mathbb{R}^p$ be the parameter space and
$\theta_0 \in \Theta$ the initialization. We denote by $f_\theta$ the
representation map and by $\hat{Y}_\theta : \mathcal{X}_S \to [K]$ the
clustering output (for any observed-view set $S$).

\paragraph{Loss family $\mathcal{F}_{rec}$.}
A reconstruction branch lies in $\mathcal{F}_{rec}$ if it admits
a per-sample decomposition
\begin{equation}
  \mathcal{L}(\theta;\mathcal{D},\mathbf{M})
  \;=\; \frac{1}{n} \sum_{i=1}^{n}
        \mathbf{1}\{|O_i| \geq 2\} \cdot
        \tilde{\ell}\bigl(\theta;\, X^{(i)},\, M^{(i)}\bigr),
  \label{eq:frec_loss}
\end{equation}
where $\tilde{\ell}$ is the per-sample reconstruction loss, defined
only when at least two views of sample $i$ are observed
(i.e., the indicator $\mathbf{1}\{|O_i| \geq 2\}$ vanishes the term
when $|O_i| < 2$). This captures cross-view reconstruction methods
in which the reconstruction loss is computed over pairs of views
co-observed for the same sample (\S\ref{sec:bg:imvc}).

\paragraph{Protocol probabilities.}
For a protocol $P$ (a distribution over $M$), define
\begin{equation}
  q_k(P) \;:=\; \mathbb{P}_{M \sim P}\bigl[|O| \geq k\bigr],
  \qquad k \in \{1,\ldots,V\}.
\end{equation}
For full-view reconstruction ($k = V$), $q_V(P) = p_c(P)$, the
complete-sample proportion. Throughout, we work with $q_2(P)$ since
$\mathcal{F}_{rec}$ requires only pairwise co-observation.

\subsubsection{Assumption}
\label{app:proofs:thm2:assumptions}

\begin{assumption}[Bounded per-sample gradient]
\label{ass:bounded_grad}
There exists $C_g < \infty$ such that for all $\theta$ in a compact
neighborhood $\mathcal{N}(\theta_0)$ containing the optimization
trajectory, and almost all $(X, M)$ with $|O| \geq 2$,
$\bigl\|\nabla_\theta \tilde{\ell}(\theta; X, M)\bigr\| \leq C_g$.
\end{assumption}

The bound $C_g$ is uniform along the optimization trajectory
and does not grow as observation support decreases. Bounded loss
values alone do not establish bounded parameter gradients.
Uniform sampling from a fixed masked dataset gives the same
argument with $q_2$ replaced by its empirical eligible fraction;
sampling only eligible examples changes the normalization.

\subsubsection{Lemma A: Gradient Magnitude Bound}
\label{app:proofs:thm2:lemA}

\begin{lemma}[Gradient magnitude under protocol $P$]
\label{lem:grad_bound}
Under Assumption~\ref{ass:bounded_grad}, for any $\theta \in \mathcal{N}(\theta_0)$
and any sample $i \sim \mathrm{Unif}([n])$ with $M^{(i)} \sim P$
independent of past iterates,
\begin{equation}
  \mathbb{E}_{i,\,M^{(i)} \sim P}
    \bigl\|\nabla_\theta \ell_i(\theta)\bigr\|
  \;\leq\; C_g \cdot q_2(P).
  \label{eq:grad_bound}
\end{equation}
\end{lemma}

\begin{proof}
By the definition of $\mathcal{F}_{rec}$ in \eqref{eq:frec_loss},
$\nabla_\theta \ell_i(\theta) = \mathbf{1}\{|O_i| \geq 2\} \cdot
\nabla_\theta \tilde{\ell}(\theta; X^{(i)}, M^{(i)})$. Therefore
\begin{align}
  \bigl\|\nabla_\theta \ell_i(\theta)\bigr\|
  &= \mathbf{1}\{|O_i| \geq 2\} \cdot
     \bigl\|\nabla_\theta \tilde{\ell}(\theta; X^{(i)}, M^{(i)})\bigr\|
     \nonumber\\
  &\leq \mathbf{1}\{|O_i| \geq 2\} \cdot C_g
     \quad \text{(Assumption~\ref{ass:bounded_grad})}.
\end{align}
Taking expectation over $i$ and $M^{(i)} \sim P$:
\begin{align}
  \mathbb{E}\bigl\|\nabla_\theta \ell_i(\theta)\bigr\|
  &\leq C_g \cdot \mathbb{P}_{M^{(i)} \sim P}\bigl[|O_i| \geq 2\bigr]
  = C_g \cdot q_2(P). \qedhere
\end{align}
\end{proof}

\subsubsection{Lemma B: Parameter Trajectory Bound}
\label{app:proofs:thm2:lemB}

\begin{lemma}[Parameter displacement after $T$ SGD steps]
\label{lem:traj_bound}
Let $\theta_t$ evolve under SGD with constant step size $\eta > 0$:
$\theta_{t+1} = \theta_t - \eta\, \nabla_\theta \ell_{i_t}(\theta_t)$,
where $\{i_t\}$ are iid uniform on $[n]$ and $\{M^{(i_t)}\} \sim P$
iid. Suppose all $\theta_t \in \mathcal{N}(\theta_0)$ for $t \in [0,T]$.
Then
\begin{equation}
  \mathbb{E}\bigl\|\theta_T - \theta_0\bigr\|
  \;\leq\; \eta T \cdot C_g \cdot q_2(P).
  \label{eq:traj_bound}
\end{equation}
\end{lemma}

\begin{proof}
By the triangle inequality,
\begin{align}
  \|\theta_T - \theta_0\|
  &= \biggl\|\sum_{t=0}^{T-1} (\theta_{t+1} - \theta_t)\biggr\|
  = \biggl\|\sum_{t=0}^{T-1} \eta \nabla_\theta \ell_{i_t}(\theta_t)\biggr\|
     \nonumber\\
  &\leq \eta \sum_{t=0}^{T-1} \bigl\|\nabla_\theta \ell_{i_t}(\theta_t)\bigr\|.
\end{align}
Taking expectation and applying Lemma~\ref{lem:grad_bound} to each term
(using that the iterate $\theta_t$ is independent of the future sample
$(i_{t}, M^{(i_t)})$ in the standard SGD analysis):
\begin{equation}
  \mathbb{E}\|\theta_T - \theta_0\|
  \leq \eta \sum_{t=0}^{T-1} \mathbb{E}\bigl\|\nabla_\theta \ell_{i_t}(\theta_t)\bigr\|
  \leq \eta T \cdot C_g \cdot q_2(P). \qedhere
\end{equation}
\end{proof}

\paragraph{Trajectory condition.}
Containment in $\mathcal{N}(\theta_0)$ is an assumption of the
bound. A small expected displacement does not guarantee that every
realized trajectory stays inside that neighborhood.

\subsubsection{Proof of Theorem~\ref{thm:fr_necessity}}
\label{app:proofs:thm2:main}

\begin{proof}
Lemma~\ref{lem:traj_bound} gives the stated displacement bound.
For a strict-complete loss, replace the support indicator in
\eqref{eq:frec_loss} by $\mathbf{1}\{|O_i|=V\}$ throughout
Lemmas~\ref{lem:grad_bound}--\ref{lem:traj_bound}. Its expectation
is $p_c$, giving $\mathbb{E}\|\theta_T-\theta_0\|
\leq\eta T C_g p_c$.
\end{proof}

\subsubsection{Scope}
\label{app:proofs:thm2:scope}

The argument applies to the gated branch with the normalization
in \eqref{eq:frec_loss}. Other objectives can update shared
parameters even when this branch is inactive. For an averaged
mini-batch gradient, the triangle inequality preserves the same
$C_g q_2$ upper bound; $1-(1-q_2)^B$ instead describes the
probability of at least one eligible sample in an iid batch of
size $B$. Adaptive optimizers require their own update analysis.

The theorem establishes an upper bound on parameter displacement.
Accuracy also depends on the initial representation and decision
margins, so small movement need not mean random clustering.
Likewise, positive observation support does not establish convergence
or an accuracy advantage over another protocol.